\documentclass[acmsmall]{acmart}

\AtBeginDocument{%
  \providecommand\BibTeX{{%
    \normalfont B\kern-0.5em{\scshape i\kern-0.25em b}\kern-0.8em\TeX}}}

\setcopyright{acmcopyright}
\copyrightyear{2022}
\acmYear{2022}

\usepackage{afterpage}
\usepackage{makecell}
\usepackage{tabulary}
\usepackage{xcolor}
\usepackage{colortbl}
\usepackage{prettyref}
\usepackage{framed}
\usepackage{booktabs}       
\usepackage[utf8]{inputenc} 
\usepackage[T1]{fontenc}    
\usepackage{hyperref}       
\usepackage{url}            
\usepackage{booktabs}       
\usepackage{amsfonts}       
\usepackage{nicefrac}       
\usepackage{microtype}      
\usepackage{xcolor}         
\usepackage{amsthm}
\usepackage{amsmath}

\usepackage{amssymb}
\usepackage{algorithm}
\usepackage{algpseudocode}
\usepackage{graphicx}
\usepackage{ulem}
\usepackage{subfigure}
\usepackage{booktabs} 
\usepackage{bm}
\usepackage{hyperref}
\usepackage{cleveref}
\usepackage{enumitem}
\usepackage{mathtools}
\usepackage{array}
\usepackage{lipsum}
\usepackage{multirow}
\usepackage{makecell}
\usepackage{xcolor}
\usepackage{tikz}
\usepackage{etoolbox}
\newcommand{\pref}[1]{\prettyref{#1}}

\newcommand{\savehyperref}[2]{\texorpdfstring{\hyperref[#1]{#2}}{#2}}
\newrefformat{eq}{\savehyperref{#1}{Eq. \textup{(\ref*{#1})}}}
\newrefformat{eqn}{\savehyperref{#1}{Eq.~\textup{(\ref*{#1})}}}
\newrefformat{lem}{\savehyperref{#1}{Lemma~\ref*{#1}}}
\newrefformat{assump}{\savehyperref{#1}{Assumption~\ref*{#1}}}
\newrefformat{defi}{\savehyperref{#1}{Definition~\ref*{#1}}}\newrefformat{tab}{\savehyperref{#1}{Table~\ref*{#1}}}
\newrefformat{lemma}{\savehyperref{#1}{Lemma~\ref*{#1}}}
\newrefformat{line}{\savehyperref{#1}{Line~\ref*{#1}}}
\newrefformat{thm}{\savehyperref{#1}{Theorem~\ref*{#1}}}
\newrefformat{corr}{\savehyperref{#1}{Corollary~\ref*{#1}}}
\newrefformat{cor}{\savehyperref{#1}{Corollary~\ref*{#1}}}
\newrefformat{sec}{\savehyperref{#1}{Section~\ref*{#1}}}
\newrefformat{app}{\savehyperref{#1}{Appendix~\ref*{#1}}}
\newrefformat{ex}{\savehyperref{#1}{Example~\ref*{#1}}}
\newrefformat{fig}{\savehyperref{#1}{Figure~\ref*{#1}}}
\newrefformat{alg}{\savehyperref{#1}{Algorithm~\ref*{#1}}}
\newrefformat{rem}{\savehyperref{#1}{Remark~\ref*{#1}}}
\newrefformat{conj}{\savehyperref{#1}{Conjecture~\ref*{#1}}}
\newrefformat{prop}{\savehyperref{#1}{Proposition~\ref*{#1}}}
\newrefformat{proto}{\savehyperref{#1}{Protocol~\ref*{#1}}}
\newrefformat{prob}{\savehyperref{#1}{Problem~\ref*{#1}}}
\newrefformat{claim}{\savehyperref{#1}{Claim~\ref*{#1}}}
\newrefformat{que}{\savehyperref{#1}{Question~\ref*{#1}}}
\newrefformat{op}{\savehyperref{#1}{Open Problem~\ref*{#1}}}
\newrefformat{fn}{\savehyperref{#1}{Footnote~\ref*{#1}}}
\newrefformat{property}{\savehyperref{#1}{Property~\ref*{#1}}}

\usepackage{xspace}
\usepackage{amsmath}
\usepackage{amsthm}
\usepackage{bbm}
\usepackage{mathrsfs}
\usepackage{bm}
\usepackage{makecell}
\usepackage{tabulary}

\newcommand{\calD}{\mathcal{D}}
\newcommand{\calL}{\mathcal{L}}

\newcommand{\calS}{\mathcal{S}}

\newcommand{\E}{\mathbb{E}}

\newcommand{\what}{\widehat}

\newcommand{\wtilde}{\widetilde}

\newcommand{\inner}[1]{\langle#1\rangle}

\newcommand{\blue}[1]{{\color{black}#1}}

\newtheorem*{theorem*}{Theorem}
\newtheorem{theorem}{Theorem}

\newtheorem{assumption}{Assumption}[section]
\newtheorem{lemma}[theorem]{Lemma}

\newtheorem*{lemma*}{Lemma}
\newtheorem{definition}{Definition}[section]

\theoremstyle{definition}

\newcommand*{\circled}[1]{\lower.7ex\hbox{\tikz\draw (0pt, 0pt)%
    circle (.5em) node {\makebox[1em][c]{\small #1}};}}
\robustify{\circled}

\newtheorem{thm}{Theorem}[section]

\newtheorem{lem}[thm]{Lemma}

\usepackage{afterpage}
\usepackage{prettyref}
\usepackage{pgfplots}
\usepackage{caption}
\usepackage{amsmath,amssymb,amsfonts}
\usepgfplotslibrary{fillbetween}
\usepackage{framed}
\newrefformat{eq}{\savehyperref{#1}{Eq. \textup{(\ref*{#1})}}}
\newrefformat{eqn}{\savehyperref{#1}{Eq.~\textup{(\ref*{#1})}}}
\newrefformat{lem}{\savehyperref{#1}{Lemma~\ref*{#1}}}
\newrefformat{lemma}{\savehyperref{#1}{Lemma~\ref*{#1}}}
\newrefformat{defi}{\savehyperref{#1}{Definition~\ref*{#1}}}
\newrefformat{line}{\savehyperref{#1}{Line~\ref*{#1}}}
\newrefformat{thm}{\savehyperref{#1}{Theorem~\ref*{#1}}}
\newrefformat{corr}{\savehyperref{#1}{Corollary~\ref*{#1}}}
\newrefformat{cor}{\savehyperref{#1}{Corollary~\ref*{#1}}}
\newrefformat{sec}{\savehyperref{#1}{Section~\ref*{#1}}}
\newrefformat{table}{\savehyperref{#1}{Table~\ref*{#1}}}
\newrefformat{app}{\savehyperref{#1}{Appendix~\ref*{#1}}}
\newrefformat{assum}{\savehyperref{#1}{Assumption~\ref*{#1}}}
\newrefformat{ex}{\savehyperref{#1}{Example~\ref*{#1}}}
\newrefformat{fig}{\savehyperref{#1}{Figure~\ref*{#1}}}
\newrefformat{alg}{\savehyperref{#1}{Algorithm~\ref*{#1}}}
\newrefformat{rem}{\savehyperref{#1}{Remark~\ref*{#1}}}
\newrefformat{conj}{\savehyperref{#1}{Conjecture~\ref*{#1}}}
\newrefformat{prop}{\savehyperref{#1}{Proposition~\ref*{#1}}}
\newrefformat{proto}{\savehyperref{#1}{Protocol~\ref*{#1}}}
\newrefformat{prob}{\savehyperref{#1}{Problem~\ref*{#1}}}
\newrefformat{claim}{\savehyperref{#1}{Claim~\ref*{#1}}}
\newrefformat{que}{\savehyperref{#1}{Question~\ref*{#1}}}
\newrefformat{op}{\savehyperref{#1}{Open Problem~\ref*{#1}}}
\newrefformat{fn}{\savehyperref{#1}{Footnote~\ref*{#1}}}

\usepackage{xspace}
\usepackage{amsmath}
\usepackage{amsthm}
\usepackage{bbm}
\usepackage{mathrsfs}
\usepackage{amssymb}
\usepackage{bm}
\usepackage{makecell}
\usepackage{tabulary}

\newcommand{\calC}{\mathcal{C}}

\theoremstyle{definition}

\begin{document}

\title{A Game-theoretic Framework for Privacy-preserving Federated Learning}

%
\author{Xiaojin Zhang}
\email{xiaojinzhang@hust.edu.cn}
\authornote{National Engineering Research Center for Big Data Technology and System, Services Computing Technology and System Lab, Cluster and Grid Computing Lab, School of Computer Science and Technology, Huazhong University of Science and Technology, Wuhan, 430074, China}
\affiliation{%
  \institution{Huazhong University of Science and Technology}
  \streetaddress{Wuhan}
  \country{China}
}

\author{Lixin Fan}
\email{lixinfan@webank.com}
\affiliation{%
  \institution{Webank}
  \city{Shenzhen}
  \country{China}
}

\author{Siwei Wang}
\email{siweiwang@microsoft.com}
\affiliation{%
  \institution{MSRA}
  \city{Beijing}
  \country{China}
}

\author{Wenjie Li}
\email{liwj20@mails.tsinghua.edu.cn}
\affiliation{%
  \institution{Tsinghua University}
  \city{Beijing}
  \country{China}
}


\author{Kai Chen}
\email{kaichen@cse.ust.hk}
\affiliation{%
  \institution{Hong Kong University of Science and Technology}
  \streetaddress{Clear Water Bay}
  \country{China}
}


\author{Qiang Yang}
\email{qyang@cse.ust.hk}
\authornote{Corresponding author}
\affiliation{%
  \institution{WeBank and Hong Kong University of Science and Technology}
  \country{China}
}






%
\renewcommand{\shortauthors}{Trovato and Tobin, et al.}

\begin{abstract}

  In federated learning, benign participants aim to optimize a global model collaboratively. However, the risk of \textit{privacy leakage} cannot be ignored in the presence of \textit{semi-honest} adversaries. Existing research has focused either on designing protection mechanisms or on inventing attacking mechanisms. While the battle between defenders and attackers seems never-ending, we are concerned with one critical question: is it possible to prevent potential attacks in advance? To address this, we propose the first game-theoretic framework that considers both FL defenders and attackers in terms of their respective payoffs, which include computational costs, FL model utilities, and privacy leakage risks. We name this game the federated learning privacy game (FLPG), in which neither defenders nor attackers are aware of all participants' payoffs.
  To handle the \textit{incomplete information} inherent in this situation, we propose associating the FLPG with an \textit{oracle} that has two primary responsibilities. First, the oracle provides lower and upper bounds of the payoffs for the players. Second, the oracle acts as a correlation device, privately providing suggested actions to each player. With this novel framework, we analyze the optimal strategies of defenders and attackers. Furthermore, we derive and demonstrate conditions under which the attacker, as a rational decision-maker, should always follow the oracle's suggestion \textit{not to attack}.
\end{abstract}



\maketitle
\section{Introduction}

In federated learning, clients aim to optimize a global model collaboratively, but risk potential leakage of private data to other parties in the face of semi-honest adversaries. Existing research has focused either on designing protection mechanisms or on inventing attacking mechanisms. In terms of privacy preservation, \textit{Randomization} \cite{geyer2017differentially,truex2020ldp,abadi2016deep}, \textit{Sparsity} \cite{shokri2015privacy,gupta2018distributed,thapa2020splitfed}, and \textit{Homomorphic Encryption (HE)} \cite{gentry2009fully,batchCryp} have been adopted to protect private training data from being disclosed. For attacking mechanisms, semi-honest attacking is one of the most common approaches, including \textit{gradient inversion attack} \cite{zhu2019dlg,zhu2020deep,geiping2020inverting,zhao2020idlg}, \textit{model inversion attack} \cite{gupta2018distributed,fredrikson2015model,he2019model,gu2021federated,thapa2020splitfed}, and \textit{GAN-based attack} \cite{hitaj2017deep,wang2019beyond}. While there seems to be a never-ending battle between various protection and attacking mechanisms (e.g., see recent survey articles \cite{lyu2020threats}), one cannot help but wonder whether it is possible to eliminate adversaries in advance or under what conditions an attacker has no incentive to launch an attack at all. Investigations along this line of thinking lead us to the game-theoretic framework presented in this article.

In our endeavor to investigate the equilibrium of a security game involving both federated learning benign participants and adversaries, we follow the typical federated learning setting and assume the existence of \textit{semi-honest} adversaries who aim to spy on private data of other parties. The payoff of the defender is measured comprehensively using model utility as the gain and privacy leakage and protection cost as the loss, while the payoff of the attacker is measured using the obtained privacy as the gain and attacking cost as the loss. It is worth noting that we take utility, privacy, and cost into account comprehensively, motivated by the theoretical analysis for the trade-off between these important quantities \cite{zhang2022trading}. The primary challenge is the absence of a suitable method for assessing the benefits of the defender, such as model performance, privacy leakage, and protection costs, and a suitable method for assessing the benefits of the attacker, such as privacy gain and attacking costs. It is not an easy task to measure their own payoffs, which depend not only on model utilities but also on privacy leakage via their strategies. However, it is unrealistic to assume that exact payoffs of adversaries are known a priori once the strategies are fixed.

To deal with these challenges, we propose a novel Federated Learning Privacy Game (FLPG) in which players, including defenders and attackers, only have access to \textit{incomplete information} about payoffs of other players as well as their own payoffs. This lack of information necessitates the assumption that each player is a rational decision-maker aiming at maximizing his/her estimated payoff based on the available information. Moreover, we propose to associate every FLPG with an \textit{oracle} whose responsibility is to provide players with \textit{lower and upper bounds of payoffs} for various choices of strategies. The oracle also plays the role of a \textit{correlation device} providing players with suggested actions based on the correlation probability of strategies among different players. This novel setting of FLPG facilitates the definition of Robust and Correlated equilibria and the analysis of conditions under which these equilibria exist.

In summary, our main contributions are as follows:

\begin{itemize}
\item We propose a novel game-theoretical framework that considers defenders and attackers in federated learning and estimates lower and upper bounds of defenders and attackers' payoffs, considering model utility, privacy leakage, and costs.

\item We measure privacy leakage from a novel perspective (please refer to \pref{eq: defi_privacy_leakage_mt}), by evaluating the gap between the data estimated by the attacker and the private data. This new definition is easy to quantify, applicable to many scenarios, and of independent research interest.

\item We derive the conditions under which the robust equilibrium is a $0$-equilibrium (\pref{thm: 0-equilibrium}), assuming the adversary is semi-honest and rational. We further derive the circumstances when the robust equilibrium is a $\tau$-equilibrium for $\tau\ge 1$ (\pref{thm: robust_equilibrium}).

\item We further investigate a setting when the oracle provides a private suggestion to each player based on the correlation probability matrix (which is also known to the players and is associated with a cost matrix), and derive the circumstances when the robust and correlated equilibrium for each player is to follow the instructions of the oracle with probability $1$, and meanwhile minimize the cost of the oracle (\pref{thm: kkt_special_solution}).
\end{itemize}

\section{Related Work}

\textbf{Federated Learning} (FL) was proposed by \cite{mcmahan2016federated,konevcny2016federated,konevcny2016federated_new} with the goal of learning machine learning models using datasets collected from multiple devices. \cite{yang2019federated} extended the notion of FL and classified the application scenarios as \textit{horizontal federated learning} (cross-device FL), \textit{vertical federated learning} (cross-silo FL), and \textit{federated transfer learning}. \cite{zhu2019deep,geiping2020inverting} demonstrated that private data could still be reconstructed from unprotected model gradient information by \textbf{semi-honest} attackers, despite the theoretical guarantee of privacy. In this paper, we consider semi-honest adversaries in horizontal FL, i.e., the attackers are honest but might infer the private data of the participants using the released parameter information.

\noindent\textbf{Game theoretic framework in privacy-preserving FL}
The works that utilize a game-theoretic framework in federated learning and take privacy leakage into account include \cite{zhan2020learning} and \cite{wu2020privacy}. \cite{zhan2020learning} investigated the motivation for edge nodes to participate in training in federated learning. \cite{wu2020privacy} formulated privacy attack and defense as a privacy-preserving attack and defense game. They measured the payoff of the defender as the privacy leakage, the payoff of the attacker as the privacy gain, and treated the problem as a zero-sum game. However, these works are fundamentally different from our FLPG setting in three ways. First, they did not consider the realistic setting of incomplete information about payoffs. Second, the given payoffs ignored the cost of the attacker, which is an important factor that influences the attacker's utility. Moreover, they did not comprehensively consider key factors, such as model performance, privacy, and efficiency, for the utility of the defender. Third, they did not investigate how to forestall potential attacks in advance, which is a key concern in our work.






\noindent\textbf{Game with incomplete information}
To address the issue that information in one-stage games is incomplete, distinct approaches were proposed. \cite{aghassi2006robust} proposed a novel concept of equilibrium called "robust-optimization equilibrium", and demonstrated that this type of equilibrium exists for any finite game with bounded payoffs. \cite{harsanyi1967games} assumed that the players' payoffs are sampled from a probability distribution, and extended Nash's results with complete information (\cite{nash1950equilibrium,nash1951non}) to settings with incomplete information.

\noindent\textbf{Correlated Equilibrium}
\cite{aumann1974subjectivity} investigated a game in which a \textit{correlation device} is introduced to model joint probabilities of players' actions. Consequently, the authors proved the existence of a more general equilibrium than Nash equilibrium, known as correlated equilibrium. Building on this line of research, \cite{correia2019nash} proposed a new class of games called \textit{correlated games}, in which each player has the \textit{freedom to follow or not to follow} suggestions made to improve payoffs based on the correlation device. We generalize the correlated equilibrium to the correlated and robust equilibrium. 

\noindent\textbf{Other Related Works}
\cite{kargupta2007game} considered defense and attacking models in the context of multi-party privacy-preserving distributed data mining. \cite{donahue2020model} and \cite{tu2021incentive} explored the motivation for data owners to participate in federated learning from a game-theoretic perspective. \cite{wu2017game} built a game model from the viewpoint of defenders in privacy-preserving data publication and analyzed the existence of Nash equilibrium. 

\section{The Federated Learning Privacy Game (FLPG)}

\begin{table*}[!htp]
\footnotesize
  \centering
  \setlength{\belowcaptionskip}{15pt}
  \caption{Table of Notation}
  \label{table: notation}
    \begin{tabular}{ccc}
    \toprule
    Notation & Meaning & Range\cr
    \midrule
    $D$ & measures the norm of the data & $1$ \cr
    $\Delta_k$ & the protection extent of client $k$ & $\Delta\in [0,1]$\cr
    $C_a$ & the attacking extent of the attacker & $\mathbb{N}$ \cr
    $\mathcal{C}$ & the set of defenders (clients) & $-$ \cr
    \midrule
    $V_{m,k}, V_{p,k}$ & model utility and privacy leakage of client $k$  & $[0,1]$ \cr
    $C(\Delta_k), E(C_a)$ & protection cost and attacking cost & $[0,1]$ \cr
    $y$ & softening constant for normalizing $C_a$ as $E(C_a) = 1 - \frac{1}{(C_a)^y}$ & $y \ge 0$ \cr
    $\eta_{m,k}, \eta_{p,k}, \eta_{c,k}$ & the payoff reference of defender $k$ towards model utility, privacy and protection cost & $[0,1]$ \cr
    $\eta_{p,a}, \eta_{c,a}$ & the payoff reference of the attacker towards privacy gain and attacking cost & $[0,1]$  \cr
    $U_{a}, U_{k}$ & the payoff of the attacker and the defender $k$ & $[0,1]$  \cr
    $\underline U$ and $\overline U$ & the lower and upper bound of a payoff variable & $[0,1]$  \cr
    
    \midrule
    $w_{d,k}, w_{o,k}$ & the protected and original gradient of the model on client $k$ & $\mathbb{R}$\cr
    $s_{t}, s_{o}$ & the recovered (at round $t$) and original private dataset from a client & $\mathbb{R}$\cr
    \midrule
    
    $\mathcal{S}_k$ & A strategy, and the strategy set of player $k$ & $-$ \cr
    $s_k, \mathbf{s}_{-k}$ & The strategy of player $k$, and the strategy vector of all the rest players & $-$ \cr
    $\Psi$ & Denotes for robust operator, such as the infimum operator and the expectation operator & $-$ \cr
    $P_o$ & The correlation matrix of $k$ players & $[0, 1]^{k\times k}$ \cr
    $p_k^y(s), p_k^n(s)$ & The probability of player $k$ replies \textit{\underline y}es or \textit{\underline n}o to follow the instruction $s$ from the oracle & $[0,1]$ \cr
    $\tau$ & a predefined threshold of the attacking extent related to  equilibrium status. & $\mathbb{N}$ \cr
    
    \midrule
    $p$ & measures the regret bound of the optimization algorithm used by the attacker & $0<p<1$\cr
    $c_a, c_b$ & measures Lipschitz & $c_a, c_b > 0$ \cr
    $c_0, c_2$ & measures the regret bound of the optimization algorithm used by the attacker & $c_0,c_2 > 0$ \cr
    \bottomrule
    \end{tabular}
\end{table*}

\blue{
Substantial research efforts have been invested in exploring the structure of privacy breaches and the development of corresponding defensive strategies. This academic focus is manifested in the studies conducted by Zhang et al. (2023) \cite{zhang2023towards}, Du et al. (2012) \cite{du2012privacy}, and a range of other influential works \cite{zhang2022no, zhang2023trading, kang2022framework, zhang2023meta, zhang2023theoretically, asi2023robustness, zhang2023probably}. Building on this foundation, we will now delve into the specifics of the threat model and the mechanisms designed for protection.

\paragraph{Threat Model} 

We delve into the threat model wherein the server takes on the role of a semi-honest attacker, while the clients function as defenders. In this context, the server is characterized as an honest-but-curious entity with the ability to inspect all messages received from clients but lacking the capacity to tamper with the training process. This threat model originates from the core concern in federated learning, where the server has the potential to illicitly glean valuable insights from clients' private datasets through gradient reverse engineering \cite{zhu2019dlg}. It is a widely accepted and representative scenario within the federated learning community \cite{kairouz2019advances}. To provide a more detailed understanding of this threat model, we outline the attacker's objectives, capabilities and knowledge.

Attacker's Objectives: We consider the server, acting as the attacker, to have two primary objectives. First, the attacker aims to accurately recover the private data of any client, by scrutinizing the information shared by that specific client. Second, the attacker seeks to estimate the cost of the attack, which is determined by the usability of the recovered data in specific tasks. This second objective essentially measures the practical value of the compromised data, with higher usability indicating a higher cost. Collectively, these two objectives define the attacker's payoff. The first objective pertains to the privacy violation aspect, while the second focuses on quantifying the attack's cost.

Attacker's Capabilities: We model the server as a semi-honest (honest-but-curious) attacker, meaning it faithfully adheres to the federated learning protocol while attempting to recover the private data of specific clients through the analysis of the information they share \cite{kairouz2019advances}. The attacker's actions do not deviate from the federated learning protocol itself, which means it does not send malicious global models to clients or alter the model architecture. Instead, the attacker may launch an independent privacy attack on the exposed gradients to infer clients' private data, all without disrupting the original federated learning process.

Attacker's Knowledge: In this scenario, we consider the common horizontal federated learning setting, such as the training protocol employed in FedAvg \cite{mcmahan2017communication}, where all clients send their gradients to the server. Consequently, we assume that the server possesses knowledge of each client's model architecture, gradients, and model parameters. It's worth noting that clients typically employ certain protection mechanisms, such as randomization and homomorphic encryption, to secure the gradients they share with the server and safeguard their data privacy.
}


\paragraph{Protection Extent} In FL, client $k$ selects a protection mechanism $M_k$, which maps the original gradient to a protected gradient. The \textit{protection extent} for client $k$ generally measures the discrepancy between the original gradient and the protected gradient, corresponding to the \textit{strategy of the defender}, and is denoted as $\Delta_k$. That is,
\begin{align}
    \Delta_k = ||w_{d,k} - w_{o,k}||,
\end{align}
\blue{
where $w_{d,k}$ represents the protected gradient of client $k$, and $w_{o,k}$ represents the original gradient of client $k$.
}


For example, the protection extent depends on the added noise for the randomization mechanism, the compression probability for the compression mechanism, and the magnitude of the key for the Paillier mechanism. After observing the protected parameter, a semi-honest adversary attempts to infer private information using optimization approaches. 

The parameters satisfying a given protection extent might not be unique. Given the original gradient, the parameter set with the same distortion extent is referred to as an equivalence class for the defender.
It is conceivable that different protection mechanisms with distinct protection parameters correspond to the same protection extent (i.e., they belong to the same equivalence class).


\paragraph{Attacking Extent} We focus on scenarios where the attacker uses an optimization approach to infer the private information of the defender. The \textit{attacking extent} depends on the optimization approach adopted by the attacker and the number of rounds used to infer the private information. We are interested in the impact of the number of rounds used for attacking on the payoff for each fixed optimization strategy. The number of rounds for attacking corresponds to the \textit{strategy of the attacker}, which is denoted as $C_a$ and is also referred to as the attacking extent.\\



\blue{

With the formal definitions of the \textit{protection extent} and the \textit{attacking extent}, we are now ready to define the payoff of the player. The privacy leakage is an important factor that is related to the payoff of both the defender and the attacker, which is measured using the gap between the estimated dataset and the original dataset.

\begin{definition}[Privacy Leakage]\label{defi: privacy_leakage}
Assume that the semi-honest attacker uses an optimization algorithm to infer the original dataset of the client based on the released parameter.
Let $s_o$ represent the original private dataset of the client, $s_t$ represent the dataset inferred by the attacker, and $C_a$ represent the total number of rounds for inferring the dataset. The privacy leakage $V_{p}$ is defined as
\begin{equation}\label{eq: defi_privacy_leakage_mt}
V_{p}=\left\{
\begin{array}{cl}
\frac{D - \frac{1}{C_a}\sum_{t = 1}^{C_a} ||s_t - s_o||}{D}, &  C_a>0\\ 
0,  &  C_a = 0\\
\end{array} \right.
\end{equation}
\textbf{Remark:}\\
(1) We assume that $||s_t - s_o||\in [0,D]$. Therefore, $V_{p}\in [0,1]$.\\
(2) When the adversary does not attack ($C_a = 0$), the privacy leakage $V_{p} = 0$.
\end{definition}

}

\begin{definition}[Payoff of the defender]
Let $\eta_{m,k}, \eta_{p,k}$ and $\eta_{c,k}$ represent the payoff preference towards model utility $V_{m,k}$, privacy leakage $V_{p,k}$, and protection cost $C(\cdot)$. The payoff of defender $k$ is
\begin{align*}
    U_k (\bm{\Delta}, C_a) = \eta_{m,k} V_{m,k}(\bm{\Delta}) - \eta_{p,k} V_{p,k}(\Delta_k, C_a) - \eta_{c,k} C(\Delta_k),
\end{align*}
where $\bm{\Delta} = (\Delta_1, \cdots, \Delta_K)$, $\Delta_k$ represents the protection extent of client $k$, and $C(\Delta_k)$ represents the protection cost of the defender due to a protection extent of $\Delta_k$.

\textbf{Remark:}\\
\blue{(1) $\Delta_k$ is defined as the distortion extent between the exposed parameter $w_{d,k}$ and the original parameter $w_{o,k}$, i.e., $\Delta_k = ||w_{d,k} - w_{o,k}||$.\\
(2) $\eta_{m,k}, \eta_{p,a}, \eta_{c,k}\in [0,1]$ satisfying that $\eta_{m,k} + \eta_{p,a} + \eta_{c,k} = 1$. Without loss of generality, We assume that $V_{m,k}(\cdot), V_{p,k}(\cdot), C(\cdot)\in [0,1]$ to ensure they share the same numerical magnitude. $V_{m,k}(\cdot)$ is related to task-specific evaluation metrics such as mean-squared error (MSE). $V_{p,k}(\cdot)$ measures the data leakage extent, we provide an example in \pref{defi: privacy_leakage}. $C(\cdot)$ measures the cost of the defender, which is a function of the protection extent.} \\
(3) The payoff of the defender is defined as a linear combination among the model utility, the privacy leakage, and the protection cost. The linear form of payoff function was widely used in existing literature, including \cite{kargupta2007game,zhan2020learning}.\\
(4) The hyperparameters including local training epoch, learning rate, and batch size are also directly related to the protection cost of the system and the risk of privacy leakage. We regard these hyperparameters as constants and focus on the variation of payoff with distinct protection extent and attacking extent. 


\end{definition} 

\begin{definition}[Payoff of the attacker]
Let $\mathcal{C}$ represent the set of defenders, $\eta_{p,a}$ represent the payoff preference towards privacy gain, and $\eta_{c,a}$ represent the payoff preference towards attacking cost. The payoff of the attacker is
\begin{align*}
    U_{a} (C_{a}, \bm{\Delta}) &  = \eta_{p,a}\cdot\sum _{k\in\mathcal{C}} V_{p,k}(\Delta_k, C_{a}) - \eta_{c,a}\cdot\lvert\mathcal{C}\rvert\cdot E(C_a),
\end{align*}
where $E(C_a)$ represents the attacking cost of the attacker due to inferring for $C_a$ rounds.
\end{definition}

\textbf{Remark:}\\
\blue{(1) The discrepancy between the original gradient and the protected gradient corresponds to the \textit{strategy of the defender} (denoted as $\Delta_k$). The number of rounds for attacking corresponds to the \textit{strategy of the attacker} (denoted as $C_a$).}\\
\blue{(2) $\eta_{p,a}, \eta_{c,a}\in [0,1]$ satisfying that $\eta_{p,a} + \eta_{c,a} = 1$. Without loss of generality, We assume that $V_{p,k}(\cdot), E(\cdot)\in [0,1]$ to ensure they share the same numerical magnitude. $E(\cdot)$ measures the cost of the attacker, which is a function of the attacking extent. As a instance, we use $E(C_a) = 1 - \frac{1}{(C_a)^y}$ in following sections. Here $y$ is a softening constant}\\
(3) We relax the definition of rounds by allowing $C_a$ to be any real number. Besides, we assume that $C_a\le 1/\epsilon$, where $\epsilon>0$ is a predefined threshold.\\




However, in our Federated Learning Privacy Game (FLPG), each player only has access to \textit{incomplete information} about the payoffs of other players as well as their own payoffs. This lack of information is due to the randomness in terms of model performance and privacy leakage. To address this issue, we propose to associate every FLPG with an \textit{oracle} whose responsibility is to provide players with \textit{lower and upper bounds of payoffs} for various choices of strategies.
%

\begin{definition}[FLPG]
The whole \textit{federated learning privacy game} (FLPG) is represented using a tuple $([K], \{s_k\}_{k\in [K]}, \{\overline U_k\}_{k\in [K]}, \{\underline U_k\}_{k\in [K]})$, which is formulated as follows:
\begin{itemize}
    \item $[K] = \{1,2,\cdots,K\}$ represents a set of $K$ players, one of which is a semi-honest adversary, and the others are defenders;
     \item Each player $k$ is associated with a set of strategies $\mathcal{S}_k$, the players select strategies $\mathbf S = (s_{1}, s_{2}, \cdots, s_{K})$ simultaneously. The strategies are characterized by either protection extent for defenders, or attacking extent for adversaries.
     \item The oracle (\pref{defi: an_oracle}) provides the upper bounds ($\overline U$) and lower bounds ($\underline U$) for the payoffs of the players.
\end{itemize}
\end{definition}

In FLPG, clients (defenders) upload the model parameters to the server to improve model utility while using protection mechanisms to safeguard their private information. The server acts as a semi-honest attacker, who could infer clients' private data based on the uploaded information. Clients (defenders) aim to maximize model utility and efficiency while minimizing privacy leakage. The server (attacker) seeks to maximize privacy gain while minimizing attacking costs. Notably, privacy leakage presents a conflicting goal between defenders and attackers. The payoff of a client depends on model performance, privacy leakage amount, and learning efficiency, while the payoff of an attacker relies on the gained information about private data and computational costs.

\blue{However, none of these payoffs are known or given to players beforehand. Instead, an oracle is responsible for supplying upper and lower bounds of payoffs to the players} \cite{pmlr-v70-zhou17b}, which is formally introduced as follows.

\begin{definition}[Oracle]\label{defi: an_oracle}
For a given \textit{federated learning privacy game} (FLPG) involving $K-1$ defender $\mathcal{P}$ and 1 \textit{attacker} $\mathcal{A}$, an \textit{oracle} $\mathcal{O}$ associated with the game is responsible for providing players with \textit{public signals}, consisting of
\begin{itemize}
    \item the element-wise lower and upper bounds ($\underline U$ and $\overline U$) of payoff matrix, which depends on the protection extent and the attacking extent, the protection extent measures the distortion of the exposed parameter information, and the attacking extent measures the cost of attacking;
    \item the correlation probability matrix $P_o$; 
    \item sample an outcome $\mathbf{S} = \{s_1,\cdots, s_K\}$ according to $P_o$, and provides player $k$ instruction $s_k$ as the suggested strategy privately, the player has the freedom to follow the instruction of the oracle or not.
\end{itemize}
\textbf{Remark:}\\
(1) In the remainder of this article, we assume that a given FLPG is consistently associated with a knowledgeable oracle capable of providing the lower and upper bounds as well as the correlation probability matrix. The oracle, in this context, is an abstraction of various realistic mechanisms that provide estimated outputs. For instance, as demonstrated in \pref{app: app_bounds_for_payoffs} of this paper, theoretical analysis of model utilities and security can serve as an FLPG oracle. Another possibility is to rely on the FL industrial standard\footnote{IEEE 3652.1-2020 - IEEE Guide for Architectural Framework and Application of Federated Machine Learning}, which evaluates the risks of protection methods under various settings.


(2) It is important to note that the oracle only informs player $k$ of what they are supposed to do, but does not disclose the suggestions made to the other players.

\end{definition}




With the guidance of the oracle, the player decides whether to follow the oracle's instructions or not. Suppose the oracle recommends strategy $s$ to player $k$. In this case, he/she follows the oracle's instruction with probability $p_k^y(s)$ and chooses not to follow it with probability $p_k^n(s) = 1 - p_k^y(s)$.

We denote $\Psi$ as the robust operator, $\underline U_k$ as the lower bound of the payoff of defender $k$, $\overline U_k$ as the upper bound of the payoff of defender $k$, $\Psi_{[\underline U_k,\overline U_k]}$ is referred to as the robust payoff, \blue{and $s_{-k}$ denotes the strategy vector of all players except for player $k$}. 

\begin{definition}[Robust Equilibrium in FLPG]
    A set of strategies $\mathbf{S}^* = (s_1^*, s_2^*, \cdots, s_k^*)$ is a \textit{robust equilibrium} of the FLPG if for any player $k$ it holds that $\Psi_{[\underline U_k,\overline U_k]}(s_{k}^*, s_{-k}^*)\ge \Psi_{[\underline U_k,\overline U_k]}(s_{k}, s_{-k}^*),$ where $\underline{U_k},\overline{U_k}$ represent the lower and upper bounds of the payoff of client $k$ provided by the oracle.
\end{definition}

If $\Psi$ is the infimum operator, then the robust equilibrium considers the worst case. It is equivalent to $\underline{U_k}(s_k^*, s_{-k})\ge \underline{U_k}(s_k, s_{-k}).$ The existence condition of equilibrium in this kind of robust finite game was proved by \cite{aghassi2006robust}. If $\Psi$ is the expectation operator over the uniform distribution, then the robust equilibrium is considered in the average case. It is equivalent to $\E_{U_k(s_k^*, s_{-k})\in [\underline{U_k}(s_k^*, s_{-k}), \overline{U_k}(s_k^*, s_{-k})]}{U_k}(s_k^*, s_{-k})\ge \E_{U_k(s_k, s_{-k})\in [\underline{U_k}(s_k, s_{-k}), \overline{U_k}(s_k, s_{-k})]}U_k(s_k, s_{-k})$, where the expectation is taken over the uniform distribution.

\begin{definition}[Robust and Correlated Equilibrium in FLPG]\label{defi: equilibrium_in_FLPG}
With certain probabilities, each player has the freedom to follow the instruction of the oracle. Otherwise, the player selects his/her candidate strategy. Let $p_k^y = (p_k^{y}(s_1), \cdots, p_k^{y}(s_{\lvert\mathcal{S}_k\rvert}))$, where $p_k^{y}(s_m)$ represents the probability of following the instruction of the oracle, if the oracle suggests $s_m$ to player $k$. A set of strategies $\mathbf{P}^* = (p_{1}^{*,y}, p_{2}^{*,y}, \cdots, p_{K}^{*,y})$ is a \textit{robust and correlated equilibrium} of the FLPG if for any player $k$ it holds that
\begin{align*}
    \Psi_{[\underline U_k,\overline U_k]}(p_{k}^{*,y}, p_{-k}^{*,y}, P_o)\ge \Psi_{[\underline U_k,\overline U_k]}(q_{k}^{y}, p_{-k}^{*,y}, P_o),
\end{align*}
where $\underline{U_k},\overline{U_k}$ represent the lower and upper bounds of the payoff of client $k$, $P_o$ represents the correlation probability matrix assigned by the oracle, and $p_k^y$ represents the following probability of player $k$. 
\end{definition}

Let $\tau$ be a predefined threshold, now we introduce the definition of $0$-equilibrium and $\tau$-equilibrium ($\tau\ge 1$) in FLPG.
\begin{definition}[$0$-equilibrium and $\tau$-equilibrium ($\tau\ge 1$) in FLPG]
If the robust equilibrium for the attacker is achieved when the attacking extent does not exceed a predefined threshold $\tau$, then the equilibruim is referred to as a \textit{$\tau$-equilibrium}. As a special case, if the robust equilibrium for the attacker is achieved when the attacking extent is $0$, then the equilibruim is referred to as a \textit{$0$-equilibrium}. If the robust and correlated equilibrium is achieved when each player follows the instructions of the oracle with probability $1$, then the equilibruim is also referred to as a \textit{$0$-equilibrium}.
\end{definition}

Nash's equilibrium concept and existence theorem \cite{morris1995common,nash1951non} have become pillars in the discipline and remain highly influential in game theory today. One of the techniques used for the existence proof is based on Kakutani's Fixed Point Theorem \cite{kakutani1941generalization}. \cite{aghassi2006robust} proved the existence of equilibria in robust finite games with bounded uncertainty sets and no private information. Having formalized our FLPG with incomplete payoff information, we establish the existence of equilibria in FLPG when the strategy sets are finite. Our existence analysis in FLPG is based on Kakutani's Fixed Point Theorem \cite{kakutani1941generalization} and follows the existence analysis of \cite{morris1995common,aghassi2006robust}. The existence condition of equilibrium in FLPG is illustrated in the following theorem. Please see \pref{app: existence_equilibrium_in_FLPG} for the full analysis.
\begin{thm}
For any $K$-person ($K<\infty$), simultaneous-move, one-stage FLPG, 
\begin{itemize}
    \item If the oracle provides the bounds of the payoffs, and the robust payoff of each player $k$ is continuous and concave, then FLPG has a robust Nash equilibrium.
    \item If the oracle provides the bounds of the payoffs and the correlation probability matrix, and the uncertainty set of payoff $U_k$ of each player $k$ is bounded, then FLPG has a robust and correlated equilibrium.
\end{itemize}
\end{thm}
The aforementioned theorem demonstrates the existence condition of the equilibrium for our newly proposed game, FLPG. In addition to providing the existence analysis, we also present an example illustrating the approach for finding the equilibria in \pref{sec: applications_appendix}.

\section{The Existence Condition for $\tau$-equilibrium in FLPG}\label{sec: equilibria_example}
In this section, we firstly provide robust equilibrium for FLPG with incomplete information. To facilitate the derivation of both the upper bound and the lower bound for the amount of privacy leaked to the semi-honest attacker, we focus on the scenario when the batch size (the size of the private data used for generating the gradient) is $1$. Taking the ingredients of federated learning including model utility, privacy leakage and time cost into consideration, we first provide exact definitions for the payoffs of the players. Due to the difficulty of quantifying privacy leakage accurately, we associate FLPG with an oracle, who provides bounds for the payoffs based on machine learning theory, such as optimization algorithms with optimal regret bounds.  We then derive a robust equilibrium using the provided bounds and analyze the circumstances under which the robust equilibrium is a $\tau$-equilibrium. We further analyze the setting when the oracle provides suggestions to the players according to correlation probability matrix, and derive the circumstances when the robust and correlated equilibrium is to follow the instructions of the oracle with probability $1$, and meanwhile minimize the cost of the oracle.


\subsection{Bounds for Payoff} \label{subsec:boundspayoff}

We provide an example to illustrate that the oracle could provide an estimation for the privacy leakage according to the optimization algorithms in machine learning theory. The semi-honest attacker uses an optimization algorithm to reconstruct the original dataset $s_o$ of the client given the exposed model information $w_d$. Let $s_t$ represent the reconstructed data at iteration $t$, and $g(s_t) = \frac{\partial \calL(s_t, w)}{\partial w}$ represent the gradient of the reconstructed dataset at round $t$, then the resulted loss is defined as $\calL(w_d, g(s_t)) = ||w_d - g(s_t)||^2$, which is also referred to as the regret at round $t$. Here the distance is measured using $\ell_2$ distance. Let $C_a$ represent the total number of learning rounds, the algorithm with sub-linear regret guarantees that the average regret converges to zero as $C_a$ increases. That is, $R(C_a) = O(C_a^p)$, where $0<p<1$. Let $d$ represent the dimension of the parameter, the optimization algorithms that guarantee asymptotically optimal regret are all applicable to our analysis, the examples include
\begin{itemize}
    \item AdaGrad algorithm proposed by \cite{duchi2011adaptive}, achieves an optimal $\Theta(\sqrt{C_a})$ regret bound. Specifically, the regret bound is $O(\max\{\log d, d^{1-\alpha/2}\sqrt{C_a}\})$, where $\alpha\in (1,2)$.
    \item Adam algorithm introduced by \cite{kingma2014adam}, achieves an optimal $\Theta(\sqrt{C_a})$ regret bound. Specifically, the regret bound is $O(\log d\sqrt{C_a})$ with an improved constant.
\end{itemize}

To derive the bounds for privacy leakage, we need the following assumptions.

\begin{assumption}\label{assump: Smoothness_and_Lipschitz}
   For any two data $s_1$ and $s_2$, assume that $c_a ||g(s_1) - g(s_2)||\le ||s_1 - s_2||\le c_b ||g(s_1) - g(s_2)||$, where $c_a, c_b>0$, $g(s) = \frac{\partial \calL(s, w)}{\partial w}$ represents the gradient of the data $s$.
\end{assumption}

The following assumption is natural originating from the regret bound of the optimization problem. 
\begin{assumption}\label{assump: optimization_bound}
   Assume that $c_0\cdot T^p \le \sum_{t = 1}^T ||g(s_t) - g(s_d)|| = \Theta(T^p) \le c_2\cdot T^p$, where $s_t$ represents the dataset reconstructed by the attacker at round $t$, $s_d$ represents the dataset satisfying that $g(s_d) = w_d$, $w_d$ denotes the distorted model parameter, and $g(s_t) = \frac{\partial \calL(s_t, w)}{\partial w}$ represents the gradient of the reconstructed dataset at round $t$.
\end{assumption}



Let \pref{assump: Smoothness_and_Lipschitz} and \pref{assump: optimization_bound} hold. Let $\mathcal{C}$ represent the set of defenders. With the regret bounds of the optimization algorithms, we can derive the bounds for privacy leakage, which is related to the extent of protection and the extent of attacking. Considering that the data-label pair might not be accurate or incomplete, the model utility is estimated based on the protection extent. Please see \pref{app: app_bounds_for_payoffs} for detailed analysis. Now we provide the lower and upper bounds of the payoffs.

\paragraph{Bounds for payoff of the defender}
We denote $\overline\Delta = \frac{1}{K}\sum_{k = 1}^K \Delta_k$. Let $P_{m,k}$ represent the model performance of client $k$ with original model information. Let $\eta_{m,k}, \eta_{p,k}$ and $\eta_{c,k}$ represent the payoff preference towards model utility $V_{m,k}(\cdot)$, privacy leakage $V_{p,k}(\cdot)$, and protection cost $C(\cdot)$. We denote $C_l = \frac{c_a c_0}{2c_b}\cdot {C_a}^{p-1}, C_u = \frac{2c_2 c_b}{c_a}\cdot {C_a}^{p-1}$, and denote $\underline V_{p,k} = 1 - \frac{c_b\cdot\Delta_k + c_b\cdot c_2\cdot (C_a)^{p-1}}{D}$, $\overline V_{p,k} = 1 - \frac{c_a\cdot\Delta_k + c_a\cdot c_0\cdot (C_a)^{p-1}}{4D}$, where $p\in (0,1)$.\\
For any $C_a>0$, the upper bound for the payoff of the defender is
\begin{align*}
    \overline U_k = \eta_{m,k}\cdot P_{m,k} - \eta_{p,k}\cdot \underline V_{p,k} - \eta_{c,k}\cdot C(\Delta_k).
\end{align*}
The lower bound for the payoff of the defender is
\begin{equation}\label{eq: lower_bound_for_uk}
\underline U_k =\left\{
\begin{array}{cl}
\eta_{m,k}\cdot P_{m,k}  - \eta_{m,k}\cdot\overline\Delta - \eta_{p,k} - \eta_{c,k}\cdot C(\Delta_k), &  \Delta_k\in [C_l, C_u],\\
\eta_{m,k}\cdot P_{m,k}  - \eta_{m,k}\cdot\overline\Delta - \eta_{p,k}\cdot\overline V_{p,k} - \eta_{c,k}\cdot C(\Delta_k),  &  \Delta_k\ge C_u \text{ or } \Delta_k\le C_l,\\
\end{array} \right.
\end{equation}
\textbf{Remark:}
If the attacker does not attack. That is, $C_a =0$, then
\begin{align}\label{eq: bounds_for_payoff_of_defender_without_attacks}
    \underline U_k = \overline U_k = \eta_{m,k}\cdot P_{m,k}  - \eta_{m,k}\cdot\overline\Delta - \eta_{c,k}\cdot C(\Delta_k).
\end{align}

\paragraph{Bounds for payoff of the attacker}
For any $C_a>0$, the upper bound for the payoff of the attacker is
\begin{equation}
\overline U_a =\left\{
\begin{array}{cl}
\sum _{k\in\mathcal{C}} \eta_{p,a} - \eta_{c,a}\cdot \lvert\mathcal{C}\rvert\cdot E(C_a), &  \Delta_k\in [C_l, C_u],\\
\sum _{k\in\mathcal{C}} \eta_{p,a}\cdot \overline V_{p,k} - \eta_{c,a}\cdot \lvert\mathcal{C}\rvert\cdot E(C_a),  &  \Delta_k\ge C_u \text{ or } \Delta_k\le C_l,\\
\end{array} \right.
\end{equation}
and the lower bound for the payoff of the attacker is
\begin{align}\label{eq: lower_bound_of_attacker_1}
    \underline U_a = \sum _{k\in\mathcal{C}} \eta_{p,a}\cdot \underline V_{p,k}  - \eta_{c,a}\cdot \lvert\mathcal{C}\rvert\cdot E(C_a),
\end{align}
\textbf{Remark:} If $C_a = 0$, then
\begin{align}\label{eq: lower_bound_of_attacker_2}
    \overline U_a = \underline U_a = 0.
\end{align}




With the bounds of the payoffs provided by the oracle, now we can derive the robust and correlated equilibrium of the FLPG.  

\subsection{Robust Equilibrium in FLPG}

With the estimated payoff matrix, now we are ready to analyze under what circumstances, the attacker has no intention to attack.
The attacking cost is $C_a$, and is normalized as $E(C_a) = 1 - \frac{1}{(C_a)^y}$. Let the protection cost be $C(\Delta_k) = (\Delta_k)^x$. Note that $x\in (0,1)$ corresponds to the protection mechanisms with low efficiency. For the encryption method such as HE, $x$ might be very small, which implies that the cost of encryption is rather large. For the randomization mechanism, $x$ could be very large, which implies that the cost of protection by adding noise is usually small.
With the defined attacking cost and protection cost, we first analyze the robust equilibrium strategies for the players without the suggestions of the oracle, and derive the circumstances when the robust equilibrium is a $\tau$-equilibrium. Please refer to \pref{app: relation_between_robust_equilibrium_and_X_equilibrium} for detailed analysis.

\blue{\pref{thm: 0-equilibrium} provides a sufficient and necessary condition for 0-equilibrium, which uses the property provided by \pref{lem: smallest_Delta}. \pref{lem: Nash_equilibrium_Delta} considers the scenarios for robust equilibrium. The robust equilibrium in \pref{lem: Nash_equilibrium_Delta} when $C_a\le\tau$ leads to \pref{thm: robust_equilibrium}.
}




The following lemma illustrates that the payoff of the attacker is negative for any non-zero attacking extent, when the distortion extent of the defender satisfies \pref{eq: delta_lower_bound_mt}. 
\begin{lem}\label{lem: smallest_Delta}
We denote $\what C_a = \left(\frac{y\eta_{c,a}D}{\eta_{p,a}c_b c_2(1-p)}\right)^{\frac{1}{p+y-1}}$. Assume that $\eta_{p,a}c_b c_2 (1-p)(2-p)\ge D y(y+1)\eta_{c,a} \text{ and } y> 1 - p$. If
\begin{align}\label{eq: delta_lower_bound_mt}
    \sum _{k\in\calC}\Delta_k > \frac{D}{c_b}\cdot|\calC| - c_2|\calC|(\hat C_a)^{p-1} - \frac{\eta_{c,a}|\calC|D}{\eta_{p,a}c_b} + \frac{\eta_{c,a}|\calC|D}{(\hat C_a)^y \eta_{p,a} c_b},
\end{align}
then $\forall C_a\ge 1, \underline U_a (C_{a}, \Delta)< 0$.
\end{lem}
\textbf{Remark:} Intuitively, if the distortion extent of the defender is quite large, the privacy stolen by the attacker becomes significantly small. Consequently, the attacker's payoff will always be negative once he/she launches an attack. On the other hand, if the attacker chooses not to attack, his/her payoff remains at zero. 

The following theorem depicts a necessary and sufficient condition for the scenario for achieving a $0$-equilibrium.
\begin{thm}\label{thm: 0-equilibrium}
We denote $\what C_a = \left(\frac{y\eta_{c,a}D}{\eta_{p,a}c_b c_2(1-p)}\right)^{\frac{1}{p+y-1}}$. Assume that $\eta_{p,a}c_b c_2 (1-p)(2-p)\ge D y(y+1)\eta_{c,a} \text{ and } y> 1 - p$. Then $0$-equilibrium in FLPG is achieved if and only if
\begin{align}\label{eq: constraint_negative}
    \frac{D}{c_b}\cdot|\calC| - c_2|\calC|(\hat C_a)^{p-1} - \frac{\eta_{c,a}|\calC|D}{\eta_{p,a}c_b} + \frac{\eta_{c,a}|\calC|D}{(\hat C_a)^y \eta_{p,a} c_b} < 0.
\end{align}

\end{thm}
\begin{proof}[Proof Sketch]
   If RHS of \pref{eq: delta_lower_bound_mt} is negative, from \pref{lem: smallest_Delta} we know that the payoff of the attacker is negative once he/she attacks, when the protection extents of all the defenders are zero. This implies that the optimal strategy for the attacker is not to attack. Therefore, $0$-equilibrium is achieved.

\blue{
If RHS of \pref{eq: delta_lower_bound_mt} is non-negative. We analyze according to the following two scenarios. If the protection extent is $0$, then the optimal strategy for the attacker is to attack (the payoff of the attacker is positive when the attacking extent is larger than $0$); If the protection extent is larger than $0$, then the optimal strategy for the attacker could not be not to attack. Therefore, $0$-equilibrium could not be achieved.
}

As a consequence, $0$-equilibrium in FLPG is achieved if and only if \pref{eq: constraint_negative} holds.
\end{proof}

The following lemma shows the solution for robust Nash Equilibrium.
\begin{lem}\label{lem: Nash_equilibrium_Delta}
We denote $\what C_a = \left(\frac{y\eta_{c,a}D}{\eta_{p,a}c_b c_2(1-p)}\right)^{\frac{1}{p+y-1}}$. Assume that $\eta_{p,a}c_b c_2 (1-p)(2-p)\ge D y(y+1)\eta_{c,a}, y> 1 - p$, and $\frac{D}{c_b}\cdot|\calC| - c_2|\calC|(\hat C_a)^{p-1} - \frac{\eta_{c,a}|\calC|D}{\eta_{p,a}c_b} + \frac{\eta_{c,a}|\calC|D}{(\hat C_a)^y \eta_{p,a} c_b} \ge 0$. Consider the scenario when $x\ge 1$.
   Let the protection extent of client $k$ (protector) be
\begin{equation}
\wtilde\Delta_k =\left\{
\begin{array}{cl}
0, &  \frac{\eta_{p,k}\cdot c_a}{4D} - \frac{\eta_{m,k}}{|\calC|}\le 0 \text{ and } x\ge 1,\\
\arg\max_{\Delta_k\in\{\what\Delta_k, C_l, C_u\}} \underline U_k,  &  \frac{\eta_{p,k}\cdot c_a}{4D} - \frac{\eta_{m,k}}{|\calC|}> 0 \text{ and } x>1,\\
D, &  x = 1 \text{ and } \frac{\eta_{p,k}\cdot c_a}{4D} - \frac{\eta_{m,k}}{|\calC|}> 0,\\
\end{array} \right.
\end{equation}
and the attacking extent of the server (attacker) be $\wtilde C_a = \arg\max_{C_a\in\{\lfloor\what C_a\rfloor, \lceil\what C_a\rceil\}} \underline U_a(C_a, \wtilde{\bm\Delta})$, then Nash equilibrium is achieved, where $\hat\Delta_k = \left(\frac{\frac{\eta_{p,k}\cdot c_a}{4D} - \frac{\eta_{m,k}}{|\calC|}}{x\eta_{c,k}}\right)^{1/(x-1)}$, $C_l = \frac{c_a c_0}{2c_b}\cdot\wtilde C_a^{p-1}$,
$C_u = \frac{2c_2 c_b}{c_a}\cdot \wtilde {C_a}^{p-1}$, $\underline U_k$ is introduced in \pref{eq: lower_bound_for_uk}.
\end{lem}

\textbf{Remark:} $\wtilde\Delta_k$ in \pref{lem: Nash_equilibrium_Delta} is not larger than RHS of \pref{eq: delta_lower_bound_mt}. Otherwise, the optimal strategy for the attacker is not to attack, and the optimal strategy for the defender is not to defend, which contradicts the assumption that  $\wtilde\Delta_k > \text{RHS of } \pref{eq: delta_lower_bound_mt} \ge 0$.

The following theorem shows the condition when $\tau$-equilibrium is achieved, for $\tau\ge 1$.
\begin{thm}[$\tau$-equilibrium for $\tau\ge 1$]\label{thm: robust_equilibrium}
We denote $\what C_a = \left(\frac{y\eta_{c,a}D}{\eta_{p,a}c_b c_2(1-p)}\right)^{\frac{1}{p+y-1}}$. Assume that $\eta_{p,a}c_b c_2 (1-p)(2-p)\ge D y(y+1)\eta_{c,a}, y> 1 - p$, and $\frac{D}{c_b}\cdot|\calC| - c_2|\calC|(\hat C_a)^{p-1} - \frac{\eta_{c,a}|\calC|D}{\eta_{p,a}c_b} + \frac{\eta_{c,a}|\calC|D}{(\hat C_a)^y \eta_{p,a} c_b} \ge 0$. Let $\Psi$ (introduced in \pref{defi: equilibrium_in_FLPG}) be the infinity operator. Let $C(\Delta_k) = (\Delta_k)^x$, and $E(C_a) = 1 - \frac{1}{(C_a)^y}$ ($x > 0$ and $y\ge 0$). Let $\wtilde C_a = \arg\max_{C_a\in\{\lfloor\what C_a\rfloor, \lceil\what C_a\rceil\}} \underline U_a(C_a, \wtilde{\bm\Delta})$, where $\what C_a = \left(\frac{y\eta_{c,a}D}{\eta_{p,a}c_b c_2(1-p)}\right)^{\frac{1}{p+y-1}}$. When 
\begin{align}
    \eta_{p,a}\cdot c_b c_2\cdot (1-p)\ge y\eta_{c,a} D{\tau}^{1 - p - y}, 
\end{align}
$\tau$-equilibrium is achieved. 
\end{thm}
\textbf{Remark:} This theorem provides the scenario where the robust equilibrium of the attacker is achieved when attacking extent of the attacker is less than or equal to a predefined threshold $\tau$.








\subsection{Robust and Correlated Equilibrium in FLPG}

We then analyze the effect of the oracle acting as the correlation device on achieving a $0$-equilibrium. For simplicity, we consider the scenario with one defender and one attacker.



The oracle provides upper and lower for the defenders and attackers bound using $\Delta_k$ and $C_a$, according to the machine learning theory \cite{kingma2014adam}. He/she informs the upper bound and lower bound of the payoff matrix. The oracle provides the correlation probability matrix. The player either follows the oracle and adopts the strategy recommended by the oracle, or adopts his/her own candidate strategy. 

In the following lemma, we derive the circumstances \blue{when the robust and correlated equilibrium (\pref{defi: equilibrium_in_FLPG})} for the attacker is to follow the instructions of the oracle with probability $1$, i.e., $0$-equilibrium. The detailed analysis is deferred to \pref{app: analysis_for_robust_correlated_equilibrium}. We take the third inequality as an illustrative example. It ensures that $p_a^{y}(G_a) = 1$, given that $p_d^{n}(E_d) = p_d^{n}(G_d) = 0$, and $p_d^{y}(G_d) = p_d^{y}(E_d) = 1$.
\begin{lem}\label{lem: correlated_equilibrium}
Let $E_d$ and $E_a$ represent the defender's and attacker's candidate strategies, respectively, while $G_d$ and $G_a$ denote the strategy of giving up defending or attacking. The oracle draws one of $(E_d,G_a), (G_d,E_a)$, $(E_d,E_a)$ and $(G_d,G_a)$ according to the correlation probability matrix. Let $\hat U_d = \Psi_{[\underline U_k,\overline U_k]} U_d$, and $\hat U_a = \Psi_{[\underline U_k,\overline U_k]} U_a$. If 
\begin{align}
    &(\hat U_d(G_d,G_a) - \hat U_d(E_d,G_a))\cdot\Pr[(G_d,G_a)] + (\hat U_d(G_d,E_a)- \hat U_d(E_d,E_a))\cdot \Pr[(G_d,E_a)]> 0, \label{eq: kkt_inequations_1}\\
    &(\hat U_d(E_d,G_a) - \hat U_d(G_d,G_a))\cdot\Pr[(E_d,G_a)] + (\hat U_d(E_d,E_a) - \hat U_d(G_d,E_a))\cdot\Pr[(E_d,E_a)]> 0, \label{eq: kkt_inequations_2}\\
    &(\hat U_a(G_d,G_a) - \hat U_a(G_d,E_a))\cdot \Pr[(G_d,G_a)] + (\hat U_a(E_d,G_a) - \hat U_a(E_d,E_a)) \cdot \Pr[(E_d,G_a)] >0, \label{eq: kkt_inequations_3}\\
    &(\hat U_a(E_d,E_a) - \hat U_a(E_d,G_a))\cdot\Pr[(E_d,E_a)] +   (\hat U_a(G_d,E_a) - \hat U_a(G_d,G_a))\cdot\Pr[(G_d,E_a)]> 0, \label{eq: kkt_inequations_4}
\end{align}
then the robust and correlated equilibrium is achieved when the players follow the instructions of the oracle with probability $1$. 
\end{lem}
\textbf{Remark:} This lemma provides scenarios where the attacker should follow the oracle's instructions with probability $1$. Following the instructions of the oracle always results in a higher payoff compared to not following the oracle.

The goal of the oracle is to ensure that the robust and correlated equilibrium for each player is achieved when the instructions of the oracle are followed with probability $1$, with a minimum amount of cost. We define $x_1 = \Pr[(G_d,G_a)], x_2 = \Pr[(G_d,E_a)], x_3 = \Pr[(E_d,G_a)]$, and $x_4 = \Pr[(E_d,E_a)]$, and formulate this goal as a constrained optimization problem that aims at minimizing the cost function $C = \sum_{i = 1}^4 c_i x_i$:

\begin{align}
\begin{array}{r@{\quad}l@{}l@{\quad}l}
\quad\min& \sum_{i = 1}^4 c_i x_i,\\
\text{s.t.,} & \text{the robust and correlated equilibrium is achieved when} \\ & \text{the instructions of the oracle are followed with probability 1}\\
& x_1 + x_2 + x_3 + x_4 = 1\\
& x_1, x_2, x_3, x_4 \ge 0\\
\end{array}
\end{align}

Combing \pref{lem: correlated_equilibrium} and using $a_{11},\dots,a_{42}$ to denote coefficients in inequalities \pref{eq: kkt_inequations_1}$\sim$\pref{eq: kkt_inequations_4}, the above optimization problem is further expressed as:
\begin{align}
\begin{array}{r@{\quad}l@{}l@{\quad}l}
\quad\min& \sum_{i = 1}^4 c_i x_i,\\
\text{s.t.,} & a_{11}x_1 + a_{12}x_2 > 0\\
& a_{21}x_3 + a_{22}x_4 > 0\\
&a_{31}x_1 + a_{32}x_3 > 0\\
&a_{41}x_2 + a_{42}x_4 > 0\\
& x_1 + x_2 + x_3 + x_4 = 1\\
& x_1, x_2, x_3, x_4 \ge 0\\
\end{array}
\end{align}

This is a convex optimization problem featuring both inequality and equality constraints. The Karush-Kuhn-Tucker (KKT) conditions can be employed to solve this constrained optimization problem. Out of its numerous solutions, we present a special case where $x_2 = x_4 = 0$ (consistently suggest that the attacker should give up attacking) in \pref{thm: kkt_special_solution}, with the proof detailed in \pref{thm: kkt_special_solution_detail}.

\begin{thm}\label{thm: kkt_special_solution}
    Suppose $a_{11} a_{41} a_{32} a_{22} - a_{31} a_{12} a_{21} a_{42} \ne 0$. If any of the following conditions is satisfied:
 \begin{itemize} 
     \item $v_2 a_{21}=0$ and $v_3 a_{32}=0$ 
     \item $v_1 a_{11}=0$ and $v_3 a_{31}=0$ 
     \item $v_1 a_{11} = v_2 a_{21} = 0 $ and $v_3 a_{31} a_{32} = 0$
     \item $v_1 a_{11} = v_2 a_{21} = 0, a_{31}a_{32}v_3 \ne 0$ and $a_{31} a_{32} < 0$
 \end{itemize}

 then we can acquire that $x_2 = x_4 = 0$, $x_1 + x_3 = 1$, and $x_1, x_3 \ge 0$. 
Here, $v_1, v_2, v_3$ are defined as \begin{align}
& v_1 = \frac{a_{31} a_{41} a_{21} c_4 - a_{31} a_{41} a_{22} c_3 - a_{31} a_{21} a_{42} c_2 + a_{41} a_{32} a_{22} c_1}{a_{11} a_{41} a_{32} a_{22} - a_{31} a_{12} a_{21} a_{42}} \label{eq: kkt_multiplier_solution_1a}\\
& v_2 = \frac{a_{11} a_{41} a_{32} c_4 - a_{11} a_{32} a_{42} c_2 - a_{31} a_{12} a_{42} c_3 + a_{12} a_{32} a_{42} c_1}{a_{11} a_{41} a_{32} a_{22} - a_{31} a_{12} a_{21} a_{42}} \label{eq: kkt_multiplier_solution_2a}\\
& v_3 = \frac{-a_{11} a_{41} a_{21} c_4 + a_{11} a_{41} a_{22} c_3 + a_{11} a_{21} a_{42} c_2 - a_{12} a_{21} a_{42} c_1}{a_{11} a_{41} a_{32} a_{22} - a_{31} a_{12} a_{21} a_{42}} \label{eq: kkt_multiplier_solution_3a}
\end{align}
\end{thm}
\textbf{Remark:} This theorem depicts the scenarios where the optimal solution of the oracle is to suggest that the attacker should relinquish their attack ($x_2 = x_4 = 0$) to achieve the robust and correlated equilibrium. We take $c_1, c_3 = 0$, and $c_2, c_4\neq 0$ as an illustrative example. On one hand, the oracle's recommendation for the attacker to give up is intended to minimize his/her own cost. On the other hand, if the attacker follows the oracle's instruction and gives up attacking with a probability of 1, he/she can achieve the robust and correlated equilibrium. The correlation probability matrix provided by the oracle is illustrated in \pref{table: correlation_prob_matrix}.


\begin{table*}[!htp]
\footnotesize
  \centering
  \setlength{\belowcaptionskip}{15pt}
  \caption{Correlation Probability Matrix for FLPG with One Defender and One Adversary}
  \label{table: correlation_prob_matrix}
    \begin{tabular}{ccccc}
    \toprule
     & $G_a$ (Give up Attacking) & $E_a$ (Candidate Strategy of Attacker)\cr
    \midrule\
$G_d$ (Give up Defending) & $x_1 = \Pr[(G_d,G_a)]$	 & $ x_2 = \Pr[(G_d,E_a)]$\cr
$E_d$ (Candidate Strategy of Defender) & $x_3 = \Pr[(E_d,G_a)]$ &  $x_4 = \Pr[(E_d,E_a)]$\cr
Optimal Solution of the Oracle & $x_1 + x_3 = 1$ &  $x_2 = x_4 = 0$\cr
    \bottomrule
    \end{tabular}
    
\end{table*}





\section{Discussion and Conclusion}
Our Federated Learning Privacy Game (FLPG) has two novel features. First, we propose a novel measurement for utility of the participants in federated learning by comprehensively considering model performance, privacy and efficiency. Second, an associated oracle is devised to provide \textit{bounds of unknown payoffs} and \textit{suggestions based on correlated probabilities} of each player's strategies. With our proposed FLPG, a secure $0$-equilibrium in which the attacker opts not to attack or follow the instruction of the oracle is investigated with its existence conditions showcased in our main theorems. We further provide the existence condition for a generalized $\tau$-equilibrium for $\tau\ge 1$. To our knowledge, the game-theoretic analysis as such is the first of its kind that has been applied to Federated Learning.

Last but not least, the motivation for and the analysis of Federated Learning Secure Game broaden our perspectives concerning federated learning security. The proposed framework allows us to analyze the battle between defenders and attackers from a strategical view, instead of focusing on protecting and attacking tactics only. Hopefully, this exploration will open a new avenue for future research and, in tandem with follow up work, make impactful contributions to the federated learning research.



\textbf{ACKNOWLEDGMENTS}
We thank Yongxin Tong for many helpful discussions. This work was supported by Hong Kong RGC TRS T41-603/20-R.




\bibliography{main}
\bibliographystyle{ACM-Reference-Format}

\clearpage

\newpage
\onecolumn
\appendix

\section{Notations}

\begin{table*}[!htp]
\footnotesize
  \centering
  \setlength{\belowcaptionskip}{15pt}
  \caption{Table of Notation}
  \label{table: notation}
    \begin{tabular}{ccc}
    \toprule
    Notation & Meaning & Range\cr
    \midrule\
    $p$ & measures the regret bound of the optimization algorithm used by the attacker & $0<p<1$\cr
    $c_a, c_b$ & measures Lipschitz & $c_a, c_b > 0$ \cr
    $c_0, c_2$ & measures the regret bound of the optimization algorithm used by the attacker & $c_0,c_2 > 0$ \cr
    $D$ & measures the norm of the data & $1$ \cr
    $\Delta$ & the protection extent & $\Delta\in [0,1]$\cr
    $T$ & the attacking extent, the number of rounds for optimization & $T\ge 1$ \cr
    $V_{m,k}$ & model utility of client $k$  & $[0,1]$ \cr
    $V_{p,k}$ & privacy leakage of client $k$  & $[0,1]$ \cr
    $C(\Delta_k)$ & protection cost & $[0,1]$ \cr
    $E(C_a)$ & attacking cost & $[0,1]$ \cr
    $\eta_{m,k}$ & the payoff reference of defender $k$ towards model utility & $[0,1]$ \cr
    $\eta_{p,k}$ & the payoff reference of defender $k$ towards privacy leakage & $[0,1]$ \cr
    $\eta_{c,k}$ & the payoff reference of defender $k$ towards protection cost & $[0,1]$ \cr
    $\eta_{p,a}$ & the payoff reference of the attacker towards privacy gain & $[0,1]$  \cr
    $\eta_{c,a}$ & the payoff reference of the attacker towards attacking cost & $[0,1]$  \cr
    \bottomrule
    \end{tabular}
\end{table*}

\section{Existence Condition for Equilibrium in FLPG}\label{app: existence_equilibrium_in_FLPG}

\begin{lemma}\label{lem: existence_condition_for_Nash_equilibrium}\cite{myerson1997game}
There exists a Nash equilibrium in the game if the following conditions are satisfied.
\begin{itemize}
    \item The player set is finite.
    \item The strategy sets are closed, bounded and convex.
    \item The utility functions are continuous and quasiconcave in the strategy space.
\end{itemize}
\end{lemma}

\begin{lemma}\label{lem: continuity}
Let $\gamma_k(p_k^y,p_{-k}^y) = \Psi_{[\underline U_k,\overline U_k]}(p_{k}^{y}, p_{-k}^{y}, P_o)$, where $\Psi$ represents the robust operator. Let $p_k^y = (p_k^{y}(s_1), \cdots, p_k^{y}(s_{\lvert\mathcal{S}_k\rvert}))$, and $p_k^{y}(s_m)$ represent the probability of following the instruction of the oracle, if the oracle suggests $s_m$ to player $k$. Let $\delta = \frac{\epsilon}{4M\lvert\mathcal{S}_k\rvert}>0$. $\forall\epsilon>0$, if $||(p_k^y,p_{-k}^y) - (q_k^y,q_{-k}^y)||_{\infty}\le\delta$, then we have that 
\begin{align*}
    \lvert\gamma_k(p_k^y,p_{-k}^y) - \gamma_k(q_k^y,q_{-k}^y)\rvert\le\epsilon.
\end{align*}
\end{lemma}
\begin{proof}
From the definition of correlation probability in FLPG, the payoff $\gamma_k$ of player $k$ depends linearly on $p_k^y$,
\begin{align*}
    \gamma_k(p_k^y,p_{-k}^y)
    & = \sum_{s_k\in\calS_k} p_k^y(s_k) \sum_{t_j\in\{y,n\}}\sum_{s_{j}\in \calS_{j}} \prod_{j\neq k}p_{j}^{t_j}(s_{j}) \Psi_{[\underline U_k,\overline U_k]}(s'_k,s'_{-k})\Pr(s_k,s_{-k})\\
     & + \sum_{s_k\in\calS_k} p_k^n(s_k) \sum_{t_j\in\{y,n\}}\sum_{s_{j}\in \calS_{j}} \prod_{j\neq k}p_{j}^{t_j}(s_{j}) \Psi_{[\underline U_k,\overline U_k]}(s'_k,s'_{-k})\Pr(s_k,s_{-k})\\
    & = \sum_{s\in\mathcal{S}_k} c_s p_k^y(s) + \sum_{s\in\mathcal{S}_k} d_s p_k^y(s),
\end{align*}
where $\lvert c_s\rvert\le 2M$ and $\lvert d_s\rvert\le 2M$ since $\lvert\Psi_{[\underline U_k,\overline U_k]}(s_k,s_{-k})\rvert\le M, \forall s_k\in\mathcal{S}_k, s_{-k}\in\mathcal{S}_{-k}$.
Then we have
\begin{align}
    \lvert\gamma_k(p_k^y,p_{-k}^y) - \gamma_k(q_k^y,q_{-k}^y)\rvert&\le \sum_{s\in\mathcal{S}_k} c_s \lvert p_k^{y}(s) - q_k^{y}(s)\rvert + \sum_{s\in\mathcal{S}_k} d_s \lvert p_k^{y}(s) - q_k^{y}(s)\rvert\\
    &\le \delta\cdot 4M\lvert\mathcal{S}_k\rvert\\
    & = \epsilon.
\end{align} 
\end{proof}

\begin{definition}[\cite{kakutani1941generalization}]
 $\Psi$ is upper semi-continuous if
\begin{align*}
    & y_n \in \Psi(x_n), n = 1,2,3,\cdots\\
    &\lim_{n\rightarrow\infty} x_n = x,\\
    &\lim_{n\rightarrow\infty} y_n = y,
\end{align*}
imply that $y\in\Psi(x)$.
\end{definition}

\begin{theorem}[Kakutani's fixed point theorem]\label{thm: fixed_point_thm}
If $W$ is a closed, bounded, and convex set in a Euclidean space, and
$\Psi$ is an upper semi-continuous point-to-set mapping of $W$ into the family of closed, convex subsets of $W$, then $\exists x\in W$ $s.t.$ $x\in\Psi(x)$.
\end{theorem}

Now we derive the existence condition of robust and correlated equilibrium for FLPG.
\begin{theorem}\label{thm: app_thm_exist_robust}
For any $K$-person ($K<\infty$), simultaneous-move, one-stage FLPG, 
\begin{itemize}
    \item If the oracle provides the bounds of the payoffs, and the robust payoff of each player $k$ is continuous and concave, then FLPG has an equilibrium.
    \item If the oracle provides the bounds of the payoffs and the correlation probability matrix, and the uncertainty set of payoff $U_k$ of each player $k$ is bounded, then FLPG has an equilibrium.
\end{itemize} 

\end{theorem}
\begin{proof}
The player set is $[K] = \{1,2,\cdots, K\}$ and is finite. The strategy of the defender is the protection extent $\Delta_k\in [0,D]$, and the strategy of the attacker is the attacking extent $C_a\in [0, 1/\epsilon]$, which is closed, bounded and convex. The robust payoffs are continuous and concave from the statement of the first part of \pref{thm: app_thm_exist_robust}. From \pref{lem: existence_condition_for_Nash_equilibrium}, FLPG has an equilibrium if the oracle provides the bounds of the payoffs, and the robust payoff of each player $k$ is continuous and concave,. Now we focus on the second part. 

Let $\Psi$ be the robust operator, $\gamma_k(p_k^y,p_{-k}^y) = \Psi_{[\underline U_k,\overline U_k]}(p_{k}^{y}, p_{-k}^{y}, P_o)$, where $p_k^y = (p_k^{y}(s_1), \cdots, p_k^{y}(s_{\lvert\mathcal{S}_k\rvert}))$, and $p_k^{y}(s_m)$ represents the probability of following the instruction of the oracle, if the oracle suggests $s_m$ to player $k$. 
From the definition of correlation probability in FLPG, the payoff $\gamma_k$ of player $k$ depends linearly on $p_k^y$,
\begin{align}
    \gamma_k(p_k^y,p_{-k}^y) = \sum_{s\in\mathcal{S}_k} c_s p_k^y(s) + c,
\end{align}
where $\lvert c_s\rvert\le 2M$, and $c\le 2M\lvert\mathcal{S}_k\rvert$.

Therefore, $\gamma_k$ is concave.
From \pref{lem: continuity}, for any $\epsilon > 0$, there exists a $\delta = \frac{\epsilon}{2M\lvert\mathcal{S}_k\rvert}>0$, such that for all $q_{k}^{y}$, $||(p_k^y,p_{-k}^y) - (q_k^y,p_{-k}^y)||_{\infty}\le\delta$ implies that $\lvert\gamma_k(p_k^y,p_{-k}^y) - \gamma_k(q_{k}^{y}, p_{-k}^{y})\rvert\le\epsilon.$ Therefore, $\gamma_k$ is continuous.

Define the mapping as
\begin{align*}
    \Xi (p_1^y,\cdots, p_K^y) = \{(q_1^y,\cdots,q_K^y)\in W\vert q_k^y\in\arg\max_{u_k\in [0,1]^{\lvert\mathcal{S}_k\rvert}}\gamma_k(u_k, p_{-k}^y), k = 1, \cdots, K\}.
\end{align*}


Notice that $W$ is a closed, bounded, and convex set in the Euclidean space.  

First, we show that $\Xi (p_1^y,\cdots, p_K^y)\neq\emptyset$. For any $p_{-k}^y$, $\gamma_k(p_k^y,p_{-k}^y)$ is continuous on $[0,1]^{\lvert\mathcal{S}_k\rvert}$. Notice that $[0,1]^{\lvert\mathcal{S}_k\rvert}$ is a nonempty, closed and bounded subset of $\mathbb {R}^{\lvert\mathcal{S}_k\rvert}$. From the extreme value theorem, a continuous function from a non-empty compact space to a subset of the real numbers attains a maximum, we have that
\begin{align}
    \arg\max_{u_k\in [0,1]^{\lvert\mathcal{S}_k\rvert}}\gamma_k(u_k, p_{-k}^y)\neq\emptyset.
\end{align}

Then, we show that $\Xi (p_1^y,\cdots, p_K^y)$ is a convex set.
Assume that $(a_1^y, \cdots, a_{K}^y)$ and $(b_{1}^y, \cdots, b_{K}^y)\in\Xi(s_1,\cdots, s_k)$. Then, for any $c_k^y\in [0,1]^{\lvert\mathcal{S}_k\rvert}$, from the definition of $\Xi (p_1^y,\cdots, p_K^y)$, we have that
\begin{align}
    \gamma_k(a_k^y, p_{-k}^y) = \gamma_k(b_k^y, p_{-k}^y)\ge\gamma_k(c_k^y, p_{-k}^y).
\end{align}

Therefore, for any $\theta\in [0,1]$, we have
\begin{align}
    \theta\cdot\gamma_k(a_k^y, p_{-k}^y) + (1-\theta)\cdot\gamma_k(b_k^y, p_{-k}^y)\ge\gamma_k(c_k, p_{-k}^y).
\end{align}

From the concavity of $\gamma_k(s_k)$, we have that
\begin{align*}
    \theta\cdot (a_1^y,\cdots, a_{K}^y) + (1-\theta)\cdot (b_{1}^y,\cdots, b_{K}^y)\in\Xi (p_1^y,\cdots, p_K^y).
\end{align*}

Therefore, $\Xi (p_1^y,\cdots, p_K^y)$ is a convex set.

Next we show $\Xi$ is an upper semi-continuous point-to-set mapping.

Assume that
\begin{align*}
    & (a_{1,n}^y,\cdots, a_{K,n}^y) \in \Xi(p_{1,n}^y, \cdots, p_{K,n}^y),\\
    &\lim_{n\rightarrow\infty} p_{k,n}^y = p_k^y, \forall k\in [K]\\
    &\lim_{n\rightarrow\infty} a_{k,n}^y = a_k^y, \forall k\in [K].
\end{align*}

From the definition of $\Xi$, we know that for any $c_{k}^y\in[0,1]^{\lvert\mathcal{S}_k\rvert}$, we have that
\begin{align*}
    \gamma_k(a_{k,n}^y, p_{-k,n}^y)\ge\gamma_k(c_{k}^y, p_{-k,n}^y).
\end{align*}

Taking the limit and using the property of continuous functions, we have that
\begin{align*}
 \gamma_k(a_{k}^y, p_{-k}^y) &= \gamma_k(\lim_{n\rightarrow\infty} a_{k,n}^y, \lim_{n\rightarrow\infty} p_{-k,n}^y)\\ 
 &= \lim_{n\rightarrow\infty} \gamma_k(a_{k,n}^y, p_{-k,n}^y)\\
&\ge\lim_{n\rightarrow\infty}\gamma_k(c_{k}^y, p_{-k,n}^y)\\
& = \gamma_k(c_{k}^y, \lim_{n\rightarrow\infty} p_{-k,n}^y)\\
& = \gamma_k(c_{k}^y, p_{-k}^y).
\end{align*}

Therefore, we have that
\begin{align*}
    (a_{1}^y,\cdots, a_{K}^y) \in \Xi(p_{1}^y, \cdots, p_{K}^y).
\end{align*}

Therefore, $\Xi$ is upper semi-continuous. From Kakutani's fixed point theorem (\pref{thm: fixed_point_thm}) on the mapping $\Xi$, we show that FLPG has an equilibrium, if the oracle provides the bounds of the payoffs and the correlation probability matrix, and the uncertainty set of payoff $U_k$ of each player $k$ is bounded.
\end{proof}

\section{Bounds for Payoffs}\label{app: app_bounds_for_payoffs}
To provide estimation for the payoff of the defenders, the oracle needs to estimate the model utility of the defender and the privacy leakage of the defender. We provide an example illustrating the estimation for the privacy leakage in \pref{sec: bound_for_privacy_leakage}, and we further provide an example illustrating the estimation for the model utility in \pref{sec: bound_for_model_performance}. 

\subsection{Bounds for Privacy Leakage}\label{sec: bound_for_privacy_leakage}

For any two datasets $s_1$ and $s_2$, assume that $c_a ||g(s_1) - g(s_2)||\le ||s_1 - s_2||\le c_b ||g(s_1) - g(s_2)||$, and $c_0\cdot T^p \le \sum_{t = 1}^T ||g(s_t) - g(s_d)|| = \Theta(T^p) \le c_2\cdot T^p$, where $s_t$ represents the dataset reconstructed by the attacker at round $t$, $s_d$ represents the dataset satisfying that $g(s_d) = w_d$, and $g(s_t) = \frac{\partial\calL(s_t, w)}{\partial w}$ represent the gradient of the reconstructed dataset at round $t$. Let $D$ be a positive constant satisfying that $||s_t- s_o||\in [0,D]$ and $c_b + c_b c_2\le D$, and $\frac{2c_2 c_b}{c_a}\le D$. Let $\mathcal{C}$ represent the set of defenders.

The privacy leakage is measured using the gap between the estimated dataset and the original dataset.

\begin{definition}[Privacy Leakage]\label{defi: privacy_leakage_app}
Assume that the semi-honest attacker uses an optimization algorithm to infer the original dataset of client $k$ based on the released parameter.
Let $s_o$ represent the original private dataset, $s_t$ represent the dataset inferred by the attacker, and $C_a$ represent the total number of rounds for inferring the dataset. The privacy leakage $V_{p}$ is defined as
\begin{equation}\label{eq: defi_privacy_leakage_mt}
V_p=\left\{
\begin{array}{cl}
1 - \frac{1}{D}\cdot\frac{1}{C_a}\sum_{t = 1}^{C_a}||s_t - s_o||, &  C_a > 0\\ 
0,  &  C_a = 0\\
\end{array} \right.
\end{equation}
\textbf{Remark:}\\
(1) We assume that $||s_t- s_o||\in [0,D]$. Therefore, $V_p\in [0,1]$.\\
(2) When the adversary does not attack ($C_a = 0$), the privacy leakage $V_p = 0$.
\end{definition}

With the regret bounds of the optimization algorithms established, we are now ready to derive bounds for privacy leakage, following the analyses in \cite{zhang2023probably} and \cite{zhang2023theoretically}.

\begin{lemma}\label{lem: bound_for_privacy_leakage}
Let $\underline V_{p}$ denote the lower bound of the privacy leakage, and $\overline V_{p}$ denote the upper bound of the privacy leakage. Assume that the semi-honest attacker uses an optimization algorithm to infer the original dataset of client $k$ based on the released parameter $w_d$. Let $\Delta = ||w_d - w_o||$ represent the distortion of the parameter, where $w_o$ represents the original parameter, and $w_d$ represents the protected parameter. The regret of the optimization algorithm in a total of $T$ rounds is $\Theta(T^p)$. If $T = 0$, we have that 
\begin{align}
    \underline V_{p} = \overline V_{p} = 0.
\end{align}
For any $T>0$,
\begin{align}
    \underline V_{p} = 1 - \frac{c_b\cdot\Delta + c_b\cdot c_2\cdot T^{p-1}}{D},
\end{align}

and

\begin{equation}
\overline V_{p} =\left\{
\begin{array}{cl}
1, &  \Delta\in [\frac{c_a c_0}{2c_b}\cdot T^{p-1}, \frac{2c_2 c_b}{c_a}\cdot T^{p-1}],\\
1 - \frac{c_b\cdot\Delta + c_b\cdot c_2\cdot T^{p-1}}{4D},  &  \Delta\ge\frac{2c_2 c_b}{c_a}\cdot T^{p-1} \text{ or } \Delta\le\frac{c_a c_0}{2c_b}\cdot T^{p-1}.\\
\end{array} \right.
\end{equation}
where $c_2\cdot T^p$ corresponds to the regret bound.
\end{lemma}

\begin{proof}
 The privacy leakage $V_{p}$ is defined as
\begin{equation}
V_p=\left\{
\begin{array}{cl}
1 - \frac{1}{D}\cdot\frac{1}{T}\sum_{t = 1}^{T}||s_t - s_o||, &  T > 0\\ 
0,  &  T = 0\\
\end{array} \right.
\end{equation}

To protect privacy, client $k$ selects a protection mechanism $M_k$, which maps the original parameter $w_{o}$ to a protected parameter $w_{d}$. After observing the protected parameter, a semi-honest adversary infers the private information using the optimization approaches. Let $s_t$ represent the reconstructed data at iteration $t$ using the optimization algorithm. Let $s_d$ be the dataset satisfying that $g(s_d) = w_d$, where $g(s) = \frac{\partial\calL (w,s)}{\partial w}$. Therefore
\begin{align*}
    R(T) & = \sum_{t = 1}^T [||g(s_t) - w_d|| - ||g(s_d) - w_d||]\\
    & = \sum_{t = 1}^T [||g(s_t) - w_d||]\\
    & = \Theta(T^p).
\end{align*}
Therefore, we have

\begin{align*}
   c_0\cdot T^p \le \sum_{t = 1}^T ||g(s_t) - g(s_d)|| = \Theta(T^p) \le c_2\cdot T^p,
\end{align*}
where $c_0$ and $c_2$ are constants independent of $T$.
From our assumption, we have that
\begin{align}
    c_a ||g(s_1) - g(s_2)||\le ||s_1 - s_2||\le c_b ||g(s_1) - g(s_2)||.
\end{align}

Then, we have that

\begin{align*}
    \sum_{t = 1}^T ||s_t- s_o||&\le \sum_{t = 1}^T ||s_d - s_o|| + \sum_{t = 1}^T ||s_t - s_d||\\
    &\le \sum_{t = 1}^T  c_b\cdot ||g(s_d) - g(s_o)|| + \sum_{t = 1}^T c_b||g(s_t) - g(s_d)||\\
    &= \sum_{t = 1}^T c_b\cdot\Delta + c_b\cdot c_2\cdot T^p,
\end{align*}
where the equality is due to $\Delta = ||g(s_d) - g(s_o)||$ is defined as the protection extent of the client. Therefore, we have that
\begin{align}\label{eq: upper_bound_of_privacy_leakage}
    \frac{1}{T}\sum_{t = 1}^T ||s_t- s_o||\le  c_b\cdot\Delta + c_b\cdot c_2\cdot T^{p-1}.
\end{align}

The privacy leakage is
\begin{align*}
    D(1-V_{p}) = \frac{1}{T}\sum_{t = 1}^{T} ||s_t - s_o||\le  c_b\cdot\Delta + c_b\cdot c_2\cdot T^{p-1}.
\end{align*}
Note that $||s_t- s_o||\in [0,D]$, and $c_b + c_b c_2\le D$. Therefore, we have that
\begin{align*}
    V_{p} \ge 1 - \frac{c_b\cdot\Delta + c_b\cdot c_2\cdot T^{p-1}}{D}.
\end{align*}
To derive the upper bound of privacy leakage, we analyze according to the following two cases.\\ 
\textbf{Case 1: $c_a\Delta\ge 2 c_2\cdot c_b T^{p-1}$.}\\
In this case, we have that
\begin{align*}
    ||s_t- s_o||&\ge \lvert||s_d - s_o|| - ||s_t - s_d||\rvert\\
    & \ge c_a\Delta - c_b ||g(s_t) - g(s_d)||,
\end{align*}
where the second inequality is due to $||s_d - s_o||\ge c_a ||g(s_d) - g(s_o)|| = c_a\Delta$ and $||s_t - s_d||\le c_b ||g(s_t) - g(s_d)||$.

\begin{align*}
    D(1-V_{p}) = \frac{1}{T}\sum_{t = 1}^{T} ||s_t - s_o||
    &\ge  c_a\Delta - c_b\cdot\frac{1}{T}\sum_{t = 1}^T||g(s_t) - g(s_d)||\\
    &\ge c_a\Delta - c_2\cdot c_b T^{p-1}\\
    &\ge\frac{1}{2}\max\{c_a\Delta, c_2\cdot c_b T^{p-1}\} \\
    &\ge \frac{c_a\Delta + c_2\cdot c_b T^{p-1}}{4}.
\end{align*}
Therefore, we have that
\begin{align*}
    V_{p} \le 1 - \frac{c_a\Delta+ c_2\cdot c_b T^{p-1}}{4D}.
\end{align*}
\textbf{Case 2: $c_a c_0\cdot T^{p-1}\ge 2 c_b\Delta$.}\\
In this case, we have that
\begin{align*}
    ||s_t- s_o||&\ge \lvert||s_t - s_d|| - ||s_d - s_o||\rvert\\
    & \ge c_a ||g(s_t) - g(s_d)|| - c_b ||g(s_d) - g(s_o)||,
\end{align*}
where the second inequality is due to $||s_d - s_o||\le c_b ||g(s_d) - g(s_o)||$ and $||s_t - s_d||\ge c_a ||g(s_t) - g(s_d)||$.

\begin{align*}
    D(1-V_{p}) = \frac{1}{T}\sum_{t = 1}^{T} ||s_t - s_o||
    &\ge  c_a\cdot\frac{1}{T}\sum_{t = 1}^T||g(s_t) - g(s_d)|| - c_b\Delta\\
    &\ge c_a\cdot c_0 T^{p-1} - c_b\Delta  \\
    &\ge\frac{1}{2}\max\{c_b\Delta, c_a\cdot c_0 T^{p-1}\} \\
    &\ge \frac{c_b\Delta + c_a\cdot c_0 T^{p-1}}{4}.
\end{align*}
Therefore, we have that
\begin{align*}
    V_{p} \le 1 - \frac{c_b\Delta + c_a\cdot c_0 T^{p-1}}{4D}.
\end{align*}
\textbf{Case 3: $\frac{c_a c_0}{2c_b}\cdot T^{p-1}\le \Delta\le \frac{2c_2 c_b}{c_a}\cdot T^{p-1}$.} 
In this case, we have
$ D(1-V_{p}) = \frac{1}{T}\sum_{t = 1}^{T} ||s_t - s_o||\ge 0$. Therefore, we have that
\begin{align}\label{eq: privacy_lower_bound}
   V_{p}\le 1.
\end{align}
In conclusion, we have that
\begin{align}
    \underline V_{p} = 1 - \frac{c_b\cdot\Delta + c_b\cdot c_2\cdot {T}^{p-1}}{D},
\end{align}
and
\begin{align}
    \overline V_{p} = 1 - \frac{c_a\cdot\Delta + c_a\cdot c_0\cdot T^{p-1}}{4D},
\end{align}
if $\Delta\ge\frac{2c_2 c_b}{c_a}\cdot T^{p-1} \text{ or } \Delta\le\frac{c_a c_0}{2c_b}\cdot T^{p-1}$.
\end{proof}

\subsection{Bounds for the Payoffs of the Players}\label{sec: bound_for_model_performance}
To derive robust equilibrium for FLPG, we first provide the bounds for the payoffs of the defenders and the attackers. In the following analysis, we denote $C_l = \frac{c_a c_0}{2c_b}\cdot {C_a}^{p-1}$, $C_u = \frac{2c_2 c_b}{c_a}\cdot {C_a}^{p-1}$, and denote $\underline V_{p,k} = 1 - \frac{c_b\cdot\Delta_k + c_b\cdot c_2\cdot (C_a)^{p-1}}{D}$, $\overline V_{p,k} = 1 - \frac{c_a\cdot\Delta_k + c_a\cdot c_0\cdot (C_a)^{p-1}}{4D}$, where $p\in (0,1)$.

\begin{lemma}\label{lem: bounds_for_payoffs_of_players}
For any $C_a>0$, the lower bound for the payoff of defender $k$ is

\begin{equation*}
\underline U_k =\left\{
\begin{array}{cl}
\eta_{m,k}\cdot P_{m,k}  - \eta_{m,k}\cdot\overline\Delta - \eta_{p,k} - \eta_{c,k}\cdot C(\Delta_k), &  \Delta_k\in (C_l, C_u),\\
\eta_{m,k}\cdot P_{m,k}  - \eta_{m,k}\cdot\overline\Delta - \eta_{p,k}\cdot\overline V_{p,k} - \eta_{c,k}\cdot C(\Delta_k),  &  \Delta_k\ge C_u \text{ or } \Delta_k\le C_l,\\
\end{array} \right.
\end{equation*}
and the upper bound of the payoff of defender $k$ is
\begin{align*}
    \overline U_k = \eta_{m,k}\cdot P_{m,k} - \eta_{p,k}\cdot \underline V_{p,k} - \eta_{c,k}\cdot C(\Delta_k).
\end{align*}
For any $C_a>0$, the lower bound for the payoff of the attacker is
\begin{align*}
    \underline U_a = \eta_{p,a}\cdot\sum _{k\in\mathcal{C}} \underline V_{p,k}  - \eta_{c,a}\cdot \lvert\mathcal{C}\rvert \cdot E(C_a),
\end{align*}
and the upper bound for the payoff of the attacker is
\begin{equation}
\overline U_a =\left\{
\begin{array}{cl}
\sum _{k\in\mathcal{C}} \eta_{p,a} - \eta_{c,a}\cdot \lvert\mathcal{C}\rvert \cdot E(C_a), &  \Delta_k\in [C_l, C_u],\\
\eta_{p,a}\cdot\sum _{k\in\mathcal{C}} \overline V_{p,k} - \eta_{c,a}\cdot \lvert\mathcal{C}\rvert \cdot E(C_a),  &  \Delta_k\ge C_u \text{ or } \Delta_k\le C_l.\\
\end{array} \right.
\end{equation}
Furthermore, if $C_a =0$, then
\begin{align}
    \overline U_k = \underline U_k = \eta_{m,k}\cdot P_{m,k}  - \eta_{m,k}\cdot\overline\Delta - \eta_{c,k}\cdot C(\Delta_k),
\end{align}
and
\begin{align}
    \overline U_a = \underline U_a = 0.
\end{align}
\end{lemma}

\begin{proof}

We first derive the bounds for the model utility of the defender. 

Recall that the payoff of defender $k$ is defined as
\begin{align*}
    U_k (s_{k}, s_{-k}) &= U_k(\Delta_{k}, s_{-k})\\
     &= \eta_{m,k}\cdot V_{m,k} - \eta_{p,k}\cdot V_{p,k} - \eta_{c,k}\cdot C(\Delta_k),
\end{align*}
and
\begin{align*}
    V_{m,k} = \left(1 - \frac{1}{\lvert\calD_k\rvert}\sum_{s\in\calD_k} \lvert\frac{1}{\lvert\mathcal{C}\rvert}\inner{\sum_{k\in\mathcal{C}} w_{d,k}, s} - l_s\rvert\right),
\end{align*}
where $\mathcal{C}$ represents the set of defenders, and $\calD_k$ represents the dataset of defender $k$.

Let $M_{k,s} = \lvert(\inner{ w_{d,k},s} - l_s) - (\inner{w_{o,k},s} - l_s)\rvert$. Then we have
\begin{align*}
    M_{k,s} & = \lvert\inner{ w_{d,k},s} - \inner{w_{o,k},s}\rvert \\
    & = \Delta_k\cdot ||s||\cdot \lvert\cos{\theta_{s,k}}\rvert\\
    &\le\Delta_k,
\end{align*}
where the inequality is due to our assumption that $||s||\le 1$.

Let $\overline\Delta = \frac{1}{\lvert\mathcal{C}\rvert}\sum_{k\in\mathcal{C}} \Delta_k$. Then, we have that
\begin{align*}
    \lvert\inner{\frac{1}{\lvert\mathcal{C}\rvert}\sum_{k\in\mathcal{C}} w_{d,k},s} - \inner{\frac{1}{\lvert\mathcal{C}\rvert}\sum_{k\in\mathcal{C}} w_{o,k},s}\rvert &= \frac{1}{\lvert\mathcal{C}\rvert}\sum_{k\in\mathcal{C}} \lvert\inner{ w_{d,k},s} - \inner{w_{o,k},s}\rvert\\
    & = \frac{1}{\lvert\mathcal{C}\rvert}\sum_{k = 1}^K M_{k,s}\\
    &\le \overline\Delta.
\end{align*}

Therefore, we have
\begin{align*}
    \lvert\inner{\frac{1}{\lvert\mathcal{C}\rvert}\sum_{k\in\mathcal{C}} w_{d,k},s} - l_s\rvert - \lvert\inner{\frac{1}{\lvert\mathcal{C}\rvert}\sum_{k\in\mathcal{C}} w_{o,k},s} - l_s\rvert
    &\le \lvert\inner{\frac{1}{\lvert\mathcal{C}\rvert}\sum_{k\in\mathcal{C}} w_{d,k},s} - \inner{\frac{1}{\lvert\mathcal{C}\rvert}\sum_{k\in\mathcal{C}} w_{o,k},s}\rvert\\
    &\le\overline\Delta.
\end{align*}

Denote 
\begin{align*}
    P_{m,k} = \left(1 - \frac{1}{\lvert\calD_k\rvert}\sum_{s\in\calD_k} \lvert\frac{1}{\lvert\mathcal{C}\rvert}\inner{\sum_{k\in\mathcal{C}} w_{o,k}, s} - l_s\rvert\right).
\end{align*}

We have that

\begin{align}
    V_{m,k} \ge P_{m,k} - \overline\Delta.
\end{align}

Let $\underline V_{m,k}$ denote the lower bound of $V_{m,k}$, and $\overline V_{m,k}$ denote the upper bound of $V_{m,k}$. Therefore, we have that

\begin{align*}
    \underline V_{m,k} = P_{m,k} - \bm{\overline\Delta},
\end{align*}

and
\begin{align*}
    \overline V_{m,k} = P_{m,k}.
\end{align*}

Recall from \pref{lem: bound_for_privacy_leakage}, we have that
\begin{align}
    \underline V_{p,k} = \underline V_{p,k},
\end{align}

and

\begin{equation}
\overline V_{p,k} =\left\{
\begin{array}{cl}
1, &  \Delta_k\in [C_l, C_u],\\
\overline V_{p,k},  &  \Delta_k\ge C_u \text{ or } \Delta_k\le C_l.\\
\end{array} \right.
\end{equation}

Now we are ready to derive bounds for payoffs of the defenders.
Let $\underline U_k$ denote the lower bound of the payoff of the defender $U_k (s_{k}, s_{-k})$. For any $C_a>0$, the lower bound of the payoff of defender $k$ is
\begin{align}
    \underline U_k & = \eta_{m,k}\cdot \underline V_{m,k} - \eta_{p,k}\cdot\overline V_{p,k} - \eta_{c,k}\cdot C(\Delta_k)\\
    & = \eta_{m,k}\cdot P_{m,k} - \eta_{m,k}\cdot\overline\Delta -  \eta_{p,k}\cdot\overline V_{p,k} - \eta_{c,k}\cdot C(\Delta_k)
\end{align}

Let $\overline U_k$ denote the upper bound of the payoff of the defender $U_k (s_{k}, s_{-k})$. For any $C_a>0$, the upper bound of the payoff of defender $k$ is

\begin{align*}
    \overline U_k = \eta_{m,k}\cdot P_{m,k} - \eta_{p,k}\cdot \underline V_{p,k} - \eta_{c,k}\cdot C(\Delta_k).
\end{align*}

If $C_a =0$, then
\begin{align}
    \overline U_k = \underline U_k = \eta_{m,k}\cdot P_{m,k}  - \eta_{m,k}\cdot\overline\Delta - \eta_{c,k}\cdot C(\Delta_k),
\end{align}

Next we derive the bounds for payoff of the attacker. Recall the payoff of the attacker is
\begin{align*}
    U_{a} (s_{a}, s_{-a}) &  = \sum _{k\in\mathcal{C}} \eta_{p,k}\cdot V_{p,k} - \eta_{c,a}\cdot \lvert\mathcal{C}\rvert \cdot E(C_a),
\end{align*}

Let $\underline U_a$ denote the lower bound of the payoff of the attacker $U_a (s_{a}, s_{-a})$. Then we have that $\underline U_a = \eta_{p,a}\cdot\sum _{k\in\mathcal{C}} \underline V_{p,k} - \eta_{c,a}\cdot \lvert\mathcal{C}\rvert \cdot E(C_a)$. For any $C_a>0$, the lower bound for the payoff of the attacker is
\begin{align*}
    \underline U_a = \eta_{p,a}\cdot\sum _{k\in\mathcal{C}} \underline V_{p,k}  - \eta_{c,a}\cdot \lvert\mathcal{C}\rvert \cdot E(C_a).
\end{align*}

Let $\overline U_a$ denote the upper bound of the payoff of the attacker $U_a (s_{a}, s_{-a})$. Then we have that $\overline U_a = \eta_{p,a}\cdot\sum _{k\in\mathcal{C}} \overline V_{p,k} - \eta_{c,a}\cdot \lvert\mathcal{C}\rvert \cdot E(C_a)$. For any $C_a>0$, the upper bound for the payoff of the attacker is

\begin{equation}
\overline U_a =\left\{
\begin{array}{cl}
\sum _{k\in\mathcal{C}} \eta_{p,a} - \eta_{c,a}\cdot \lvert\mathcal{C}\rvert \cdot E(C_a), &  \Delta_k\in [C_l, C_u],\\
\eta_{p,a}\cdot\sum _{k\in\mathcal{C}} \overline V_{p,k} - \eta_{c,a}\cdot \lvert\mathcal{C}\rvert \cdot E(C_a),  &  \Delta_k\ge C_u \text{ or } \Delta_k\le C_l.\\
\end{array} \right.
\end{equation}

If $C_a =0$, then
\begin{align}
    \overline U_a = \underline U_a = 0.
\end{align}
\end{proof}
\section{The Circumstances When the Robust Equilibrium is a $\tau$-equilibrium}\label{app: relation_between_robust_equilibrium_and_X_equilibrium}


Assume that $\Psi$ (introduced in \pref{defi: equilibrium_in_FLPG}) is the infinity operator. Then the definition of equilibrium is equivalent to
\begin{align*}
    \underline{U_k}(s_k^*)\ge \underline{U_k}(s_k),
\end{align*}
where $\underline{U_k}(s_k)$ represents the upper bound of the payoff of client $k$. The existence condition of equilibrium in this kind of robust finite game was proved by \cite{aghassi2006robust}.

Using the structural property of the utilities of the players, we simplify the calculation of the equilibrium.
We first derive the equilibrium in FLPG from the worst-case perspective.


\subsection{Analysis for \pref{lem: smallest_Delta}}

The following theorem illustrates that the payoff of the attacker is negative for any non-zero attacking extent, when the distortion extent of the defender satisfies \pref{eq: delta_lower_bound}. 
\begin{lem}\label{lem: delta_lower_bound_app}
We denote $\what C_a = \left(\frac{y\eta_{c,a}D}{\eta_{p,a}c_b c_2(1-p)}\right)^{\frac{1}{p+y-1}}$. Assume that $\eta_{p,a}c_b c_2 (1-p)(2-p)\ge D y(y+1)\eta_{c,a} \text{ and } y> 1 - p$. If
\begin{align}\label{eq: delta_lower_bound}
    \sum _{k\in\calC}\Delta_k > \frac{D}{c_b}\cdot|\calC| - c_2|\calC|(\hat C_a)^{p-1} - \frac{\eta_{c,a}|\calC|D}{\eta_{p,a}c_b} + \frac{\eta_{c,a}|\calC|D}{(\hat C_a)^y \eta_{p,a} c_b},
\end{align}
then $\forall C_a\ge 1, \underline U_a (C_{a}, \Delta)< 0$.
\end{lem}

\begin{proof}
We denote $\what\Delta_k = \left(\frac{\frac{\eta_{p,k}\cdot c_a}{4D} - \frac{\eta_{m,k}}{|\calC|}}{x\eta_{c,k}}\right)^{1/(x-1)}$, $\wtilde\Delta_k = \arg\max_{\Delta_k\in\{\what\Delta_k, C_l\}} \underline U_k$, $\ddot U_a =  \eta_{p,a}\cdot\sum _{k\in\calC} (1 - \frac{c_b + c_b\cdot c_2\cdot (\wtilde C_a)^{p-1}}{D})  - \eta_{c,a}\cdot |\calC|\cdot (1 - \frac{1}{(\wtilde C_a)^y})$.

The payoff of the attacker is expressed as
\begin{align*}
    \underline U_a(C_a, \Delta) &= \eta_{p,a}\cdot\sum _{k\in\calC} (1 - \frac{c_b\cdot\Delta_k + c_b\cdot c_2\cdot {C_a}^{p-1}}{D})  - \eta_{c,a}\cdot |\calC|\cdot(1 - \frac{1}{(C_a)^y}).
\end{align*}
For any $C_a \ge 1$ we have that
\begin{align*}
  \frac{\partial ^{2} \underline U_a}{\partial {C_a}^{2}} = \frac{|\calC|\eta_{p,a}\cdot c_b c_2\cdot (1-p)  (p-2)}{D}C_a^{p-3} + y(y+1)\eta_{c,a} |\calC| C_a^{-y-2} < 0.
\end{align*}
Assume that
\begin{align}
    \eta_{p,a}c_b c_2 (1-p)(2-p)\ge D y(y+1)\eta_{c,a} \text{ and } y> 1 - p,
\end{align}
then the second-order derivative is non-positive.
The first-order derivative is
\begin{align}
    \frac{\partial \underline U_a}{\partial {C_a}} = \frac{|\calC|\eta_{p,a}\cdot c_b c_2\cdot (1-p)}{D}C_a^{p-2} - y\eta_{c,a}|\calC| C_a^{-y-1}.
\end{align}
Setting the first-order derivative $\frac{\partial \underline U_a}{\partial {C_a}} = 0$. Then the maximal value of $\underline U_a$ is achieved when
\begin{align*}
    \what C_a = \left(\frac{y\eta_{c,a}D}{\eta_{p,a}c_b c_2(1-p)}\right)^{\frac{1}{p+y-1}}.
\end{align*}
Therefore, when $\sum _{k\in\calC}\Delta_k > \frac{D}{c_b}\cdot|\calC| - c_2|\calC|(\hat C_a)^{p-1} - \frac{\eta_{c,a}|\calC|D}{\eta_{p,a}c_b} + \frac{\eta_{c,a}|\calC|D}{(\hat C_a)^y \eta_{p,a} c_b}$,
\begin{align}
    \underline U_a(\what C_a, \Delta) < 0.
\end{align}
Therefore , $\forall C_a\ge 1$, $\underline U_a (C_{a}, \Delta)< 0$.
\end{proof}

\subsection{Analysis for \pref{thm: 0-equilibrium}}
The following theorem illustrates the condition when $0$-equilibrium is achieved.

\begin{thm}\label{thm: condition_for_0_equilibrium_app}
We denote $\what C_a = \left(\frac{y\eta_{c,a}D}{\eta_{p,a}c_b c_2(1-p)}\right)^{\frac{1}{p+y-1}}$. Assume that $\eta_{p,a}c_b c_2 (1-p)(2-p)\ge D y(y+1)\eta_{c,a} \text{ and } y> 1 - p$. 
If 
\begin{align}
    \frac{D}{c_b}\cdot|\calC| - c_2|\calC|(\hat C_a)^{p-1} - \frac{\eta_{c,a}|\calC|D}{\eta_{p,a}c_b} + \frac{\eta_{c,a}|\calC|D}{(\hat C_a)^y \eta_{p,a} c_b} < 0,
\end{align}
then $0$-equilibrium in FLPG is achieved.

If 
\begin{align}
    \frac{D}{c_b}\cdot|\calC| - c_2|\calC|(\hat C_a)^{p-1} - \frac{\eta_{c,a}|\calC|D}{\eta_{p,a}c_b} + \frac{\eta_{c,a}|\calC|D}{(\hat C_a)^y \eta_{p,a} c_b} \ge 0,
\end{align}
then $0$-equilibrium in FLPG could not be achieved.
\end{thm}

\begin{proof}
Let $\what C_a = \left(\frac{y\eta_{c,a}D}{\eta_{p,a}c_b c_2(1-p)}\right)^{\frac{1}{p+y-1}}$. Assume that $\eta_{p,a}c_b c_2 (1-p)(2-p)\ge D y(y+1)\eta_{c,a} \text{ and } y> 1 - p$. If $\frac{D}{c_b}\cdot|\calC| - c_2|\calC|(\hat C_a)^{p-1} - \frac{\eta_{c,a}|\calC|D}{\eta_{p,a}c_b} + \frac{\eta_{c,a}|\calC|D}{(\hat C_a)^y \eta_{p,a} c_b} < 0$, when the protection extents of all the defenders are zero, from \pref{lem: delta_lower_bound_app} we know that the payoff of the attacker is negative once he/she attacks. This implies that the optimal strategy for the attacker is not to attack. Therefore, $0$-equilibrium is achieved.

On the other hand, consider the scenario when $\frac{D}{c_b}\cdot|\calC| - c_2|\calC|(\hat C_a)^{p-1} - \frac{\eta_{c,a}|\calC|D}{\eta_{p,a}c_b} + \frac{\eta_{c,a}|\calC|D}{(\hat C_a)^y \eta_{p,a} c_b}$ is non-negative. If the protection extent is $0$, then the optimal strategy for the attacker is to attack. If the protection extent is larger than $0$, then the optimal strategy for the attacker could not be not to attack. Assume the optimal strategy for the attacker is not to attack. If the attacker does not attack, then the privacy leakage of the defender is zero. From \pref{eq: bounds_for_payoff_of_defender_without_attacks}, the payoff of the defender is expressed as 
\begin{align}
    \underline U_k = \eta_{m,k}\cdot P_{m,k}  - \eta_{m,k}\cdot\overline\Delta - \eta_{c,k}\cdot C(\Delta_k),
\end{align}
which decreases with the distortion extent. Therefore, the optimal strategy for the defender is not to defend, which contradicts with the assumption that the protection extent is larger than $0$. Therefore, when $\frac{D}{c_b}\cdot|\calC| - c_2|\calC|(\hat C_a)^{p-1} - \frac{\eta_{c,a}|\calC|D}{\eta_{p,a}c_b} + \frac{\eta_{c,a}|\calC|D}{(\hat C_a)^y \eta_{p,a} c_b}$ is non-negative, $0$-equilibrium could not be achieved.
\end{proof}

\subsection{Analysis for \pref{lem: Nash_equilibrium_Delta}}
The following lemma shows the solution for robust Nash Equilibrium. 
\begin{lem}
We denote $\what C_a = \left(\frac{y\eta_{c,a}D}{\eta_{p,a}c_b c_2(1-p)}\right)^{\frac{1}{p+y-1}}$. Assume that $\eta_{p,a}c_b c_2 (1-p)(2-p)\ge D y(y+1)\eta_{c,a}, y> 1 - p$, and $\frac{D}{c_b}\cdot|\calC| - c_2|\calC|(\hat C_a)^{p-1} - \frac{\eta_{c,a}|\calC|D}{\eta_{p,a}c_b} + \frac{\eta_{c,a}|\calC|D}{(\hat C_a)^y \eta_{p,a} c_b} \ge 0$. Consider the scenario when $x\ge 1$.
   Let the protection extent of client $k$ be
\begin{equation}
\wtilde\Delta_k =\left\{
\begin{array}{cl}
0, &  \frac{\eta_{p,k}\cdot c_a}{4D} - \frac{\eta_{m,k}}{|\calC|}\le 0 \text{ and } x\ge 1,\\
\arg\max_{\Delta_k\in\{\what\Delta_k, C_l, C_u\}} \underline U_k,  &  \frac{\eta_{p,k}\cdot c_a}{4D} - \frac{\eta_{m,k}}{|\calC|}> 0 \text{ and } x>1,\\
D, &  x = 1 \text{ and } \frac{\eta_{p,k}\cdot c_a}{4D} - \frac{\eta_{m,k}}{|\calC|}> 0,\\
\end{array} \right.
\end{equation}
and the attacking extent of the server be $\wtilde C_a = \arg\max_{C_a\in\{\lfloor\what C_a\rfloor, \lceil\what C_a\rceil\}} \underline U_a(C_a, \wtilde{\bm\Delta})$, then Nash equilibrium is achieved, where $\hat\Delta_k = \left(\frac{\frac{\eta_{p,k}\cdot c_a}{4D} - \frac{\eta_{m,k}}{|\calC|}}{x\eta_{c,k}}\right)^{1/(x-1)}$, $C_l = \frac{c_a c_0}{2c_b}\cdot\wtilde C_a^{p-1}$,
$C_u = \frac{2c_2 c_b}{c_a}\cdot \wtilde {C_a}^{p-1}$, $\underline U_k$ is introduced in \pref{eq: lower_bound_for_uk}.
\end{lem}
\begin{proof}

From the assumption that $\eta_{p,a}c_b c_2 (1-p)(2-p)\ge D y(y+1)\eta_{c,a}, y> 1 - p$, and $\frac{D}{c_b}\cdot|\calC| - c_2|\calC|(\hat C_a)^{p-1} - \frac{\eta_{c,a}|\calC|D}{\eta_{p,a}c_b} + \frac{\eta_{c,a}|\calC|D}{(\hat C_a)^y \eta_{p,a} c_b} \ge 0$ and \pref{thm: condition_for_0_equilibrium_app}, $0$-equilibrium in FLPG could not be achieved.

From \pref{lem: bounds_for_payoffs_of_players}, we have that

\begin{equation}
\underline U_k =\left\{
\begin{array}{cl}
\eta_{m,k}\cdot P_{k,m}  - \eta_{m,k}\cdot\overline\Delta - \eta_{p,k}  - \eta_{c,k}\cdot (\Delta_k)^x, &  \Delta_k\in (C_l, C_u),\\
\eta_{m,k}\cdot P_{k,m}  - \eta_{m,k}\cdot\overline\Delta - \eta_{p,k}\cdot\overline V_{p,k} - \eta_{c,k}\cdot (\Delta_k)^x,  &  \Delta_k\ge C_u \text{ or } \Delta_k\le C_l.\\
\end{array} \right.
\end{equation}
(2) If $x\ge 1$, we then analyze according to the following two cases.\\
\textbf{Consider interval 1: $\Delta_k\in (C_l, C_u)$}.\\
The second-order derivative
\begin{align*}
  \frac{\partial ^{2} \underline U_k}{\partial {\Delta_{k}}^{2}} = -x(x-1)\Delta_k^{x-2}\cdot \eta_{c,k} < 0,
\end{align*}
and the first-order derivative

\begin{align}
    \frac{\partial \underline U_k}{\partial {\Delta_{k}}}  = -\frac{\eta_{m,k}}{|\calC|} - \eta_{c,k}\cdot x\Delta_k^{x-1} < 0.
\end{align}
The first-order derivative of $\underline U_k< 0$. Note that $C_l\le D$ since $\frac{2c_2 c_b}{c_a}\le D$ by our assumption and that $C_l = \frac{c_a c_0}{2c_b}\cdot C_a^{p-1}\le\frac{2c_2 c_b}{c_a}\cdot C_a^{p-1}\le D$.\\
\textbf{Consider interval 2: $\Delta_k\ge C_u \text{ or } 0\le\Delta_k\le C_l$.}\\ 
In this case, the second-order derivative
\begin{align*}
  \frac{\partial ^{2}  \underline U_k}{\partial {\Delta_{k}}^{2}}& = -x(x-1)\Delta_k^{x-2}\cdot \eta_{c,k} < 0.
\end{align*}
The first-order derivative
\begin{align}
    \frac{\partial \underline U_k}{\partial {\Delta_{k}}} = -\eta_{m,k}\cdot\frac{1}{|\calC|} + \frac{\eta_{p,k}\cdot c_a}{4D} - \eta_{c,k}\cdot x\Delta_k^{x-1}.
\end{align}
(a) If $\frac{\eta_{p,k}\cdot c_a}{4D} - \frac{\eta_{m,k}}{|\calC|}\le 0$, then the first-order derivative $\frac{\partial \underline U_k}{\partial {\Delta_{k}}} \le 0$.\\
Therefore, the maximal value of $\underline U_k$ is achieved when 
\begin{align}\label{eq: delta_k_3}
   \wtilde\Delta_k = 0. 
\end{align}
(b) If $\frac{\eta_{p,k}\cdot c_a}{4D} - \frac{\eta_{m,k}}{|\calC|}> 0$, setting the first-order derivative $\frac{\partial \underline U_k}{\partial {\Delta_{k}}} = 0$.\\  
(i) If $x = 1$,
then the first-order derivative is  
\begin{align}
    \frac{\partial \underline U_k}{\partial {\Delta_{k}}} = -\eta_{m,k}\cdot\frac{1}{|\calC|} + \frac{\eta_{p,k}\cdot c_a}{4D} - \eta_{c,k}.
\end{align}

If the first-order derivative 
\begin{align}
    \frac{\partial \underline U_k}{\partial {\Delta_{k}}} = -\eta_{m,k}\cdot\frac{1}{|\calC|} + \frac{\eta_{p,k}\cdot c_a}{4D} - \eta_{c,k}\ge 0,
\end{align}

then the maximal value of $\underline U_k$ is achieved when
\begin{align}\label{eq: delta_k_5}
    \wtilde\Delta_k = D \text{ or } C_l.
\end{align}

If the first-order derivative 
\begin{align}
    \frac{\partial \underline U_k}{\partial {\Delta_{k}}} = -\eta_{m,k}\cdot\frac{1}{|\calC|} + \frac{\eta_{p,k}\cdot c_a}{4D} - \eta_{c,k}< 0,
\end{align}

then the maximal value of $\underline U_k$ is achieved when
\begin{align}\label{eq: delta_k_5}
    \wtilde\Delta_k = 0 \text{ or } C_u.
\end{align}

(ii) If $x>1$, then 
\begin{align}\label{eq: delta_k_7}
    \wtilde\Delta_k = \left(\frac{\frac{\eta_{p,k}\cdot c_a}{4D} - \frac{\eta_{m,k}}{|\calC|}}{x\eta_{c,k}}\right)^{1/(x-1)} \text{ or } C_l.
\end{align}

Let $\what\Delta_k = \left(\frac{\frac{\eta_{p,k}\cdot c_a}{4D} - \frac{\eta_{m,k}}{|\calC|}}{x\eta_{c,k}}\right)^{1/(x-1)}$. Combining \pref{eq: delta_k_3}, \pref{eq: delta_k_5} and \pref{eq: delta_k_7}, we have that

\begin{equation}
\wtilde\Delta_k =\left\{
\begin{array}{cl}
0, &  \frac{\eta_{p,k}\cdot c_a}{4D} - \frac{\eta_{m,k}}{|\calC|}\le 0 \text{ and } x\ge 1,\\
\arg\max_{\Delta_k\in\{\what\Delta_k, C_l, C_u\}} \underline U_k,  &  \frac{\eta_{p,k}\cdot c_a}{4D} - \frac{\eta_{m,k}}{|\calC|}> 0 \text{ and } x>1,\\
D, &  x = 1 \text{ and } \frac{\eta_{p,k}\cdot c_a}{4D} - \frac{\eta_{m,k}}{|\calC|}> 0\\
\end{array} \right.
\end{equation}

\noindent The payoff of the attacker is expressed as
\begin{align*}
    \underline U_a(C_a, \Delta) &= \eta_{p,a}\cdot\sum _{k\in\calC} (1 - \frac{c_b\cdot\Delta_k + c_b\cdot c_2\cdot {C_a}^{p-1}}{D})  - \eta_{c,a}\cdot |\calC|\cdot(1 - \frac{1}{(C_a)^y}).
\end{align*}
For any $C_a \ge 1$ we have that
\begin{align*}
  \frac{\partial ^{2} \underline U_a}{\partial {C_a}^{2}} = \frac{|\calC|\eta_{p,a}\cdot c_b c_2\cdot (1-p)  (p-2)}{D}C_a^{p-3} + y(y+1)\eta_{c,a} |\calC| C_a^{-y-2} < 0.
\end{align*}
Assume that
\begin{align}
    \eta_{p,a}c_b c_2 (1-p)(2-p)\ge D y(y+1)\eta_{c,a} \text{ and } y> 1 - p,
\end{align}
then the second-order derivative is non-positive.
The first-order derivative is
\begin{align}
    \frac{\partial \underline U_a}{\partial {C_a}} = \frac{|\calC|\eta_{p,a}\cdot c_b c_2\cdot (1-p)}{D}C_a^{p-2} - y\eta_{c,a}|\calC| C_a^{-y-1}.
\end{align}
Setting the first-order derivative $\frac{\partial \underline U_a}{\partial {C_a}} = 0$. Then the maximal value of $\underline U_a$ is achieved when
\begin{align*}
    \what C_a = \left(\frac{y\eta_{c,a}D}{\eta_{p,a}c_b c_2(1-p)}\right)^{\frac{1}{p+y-1}}.
\end{align*}
It requires to compare $\underline U_a(\lfloor\what C_a\rfloor, S_{-a})$ and $\underline U_a(\lceil\what C_a\rceil, S_{-a})$ since the number of rounds for learning should be an integer. Then $\wtilde C_a$ is updated as the one achieving a larger payoff. That is,
\begin{align}
    \wtilde C_a = \arg\max_{C_a\in\{\lfloor\what C_a\rfloor, \lceil\what C_a\rceil\}} \underline U_a(C_a, \wtilde{\bm\Delta}).
\end{align}
The maximal value of $\underline U_a$ is achieved when
\begin{align}
    \wtilde C_a = \arg\max_{C_a\in\{\lfloor\what C_a\rfloor, \lceil\what C_a\rceil\}} \underline U_a(C_a, \wtilde{\bm\Delta}),
\end{align}
where $\what C_a = \left(\frac{y\eta_{c,a}D}{\eta_{p,a}c_b c_2(1-p)}\right)^{\frac{1}{p+y-1}}$.\\
\end{proof}







\subsection{Analysis for \pref{thm: robust_equilibrium}}
The following corollary shows the condition when $\tau$-equilibrium is achieved, for $\tau\ge 1$.
\begin{thm}[$\tau$-equilibrium for $\tau\ge 1$]
We denote $\what C_a = \left(\frac{y\eta_{c,a}D}{\eta_{p,a}c_b c_2(1-p)}\right)^{\frac{1}{p+y-1}}$. Assume that $\eta_{p,a}c_b c_2 (1-p)(2-p)\ge D y(y+1)\eta_{c,a}, y> 1 - p$, and $\frac{D}{c_b}\cdot|\calC| - c_2|\calC|(\hat C_a)^{p-1} - \frac{\eta_{c,a}|\calC|D}{\eta_{p,a}c_b} + \frac{\eta_{c,a}|\calC|D}{(\hat C_a)^y \eta_{p,a} c_b} \ge 0$. Let $\Psi$ (introduced in \pref{defi: equilibrium_in_FLPG}) be the infinity operator. Let $C(\Delta_k) = (\Delta_k)^x$, and $E(C_a) = 1 - \frac{1}{(C_a)^y}$ ($x > 0$ and $y\ge 0$). Let $\wtilde C_a = \arg\max_{C_a\in\{\lfloor\what C_a\rfloor, \lceil\what C_a\rceil\}} \underline U_a(C_a, \wtilde{\bm\Delta})$, where $\what C_a = \left(\frac{y\eta_{c,a}D}{\eta_{p,a}c_b c_2(1-p)}\right)^{\frac{1}{p+y-1}}$. When 
\begin{align}
    \eta_{p,a}\cdot c_b c_2\cdot (1-p)\ge y\eta_{c,a} D{\tau}^{1 - p - y}, 
\end{align}
then $\tau$-equilibrium is achieved. 
\end{thm}

\begin{proof}
From the assumption that $\eta_{p,a}c_b c_2 (1-p)(2-p)\ge D y(y+1)\eta_{c,a}, y> 1 - p$, and $\frac{D}{c_b}\cdot|\calC| - c_2|\calC|(\hat C_a)^{p-1} - \frac{\eta_{c,a}|\calC|D}{\eta_{p,a}c_b} + \frac{\eta_{c,a}|\calC|D}{(\hat C_a)^y \eta_{p,a} c_b} \ge 0$ and \pref{thm: condition_for_0_equilibrium_app}, $0$-equilibrium in FLPG could not be achieved.

The payoff of the attacker is expressed as
\begin{align*}
    \underline U_a(C_a, \Delta) &= \eta_{p,a}\cdot\sum _{k\in\calC} (1 - \frac{c_b\cdot\Delta_k + c_b\cdot c_2\cdot {C_a}^{p-1}}{D})  - \eta_{c,a}\cdot |\calC|\cdot(1 - \frac{1}{(C_a)^y}).
\end{align*}
For any $C_a \ge 1$ we have that
\begin{align*}
  \frac{\partial ^{2} \underline U_a}{\partial {C_a}^{2}} = \frac{|\calC|\eta_{p,a}\cdot c_b c_2\cdot (1-p)  (p-2)}{D}C_a^{p-3} + y(y+1)\eta_{c,a} |\calC| C_a^{-y-2} < 0.
\end{align*}
Assume that
\begin{align}
    \eta_{p,a}c_b c_2 (1-p)(2-p)\ge D y(y+1)\eta_{c,a} \text{ and } y> 1 - p,
\end{align}
then the second-order derivative is non-positive.
The first-order derivative is
\begin{align}
    \frac{\partial \underline U_a}{\partial {C_a}} = \frac{|\calC|\eta_{p,a}\cdot c_b c_2\cdot (1-p)}{D}C_a^{p-2} - y\eta_{c,a}|\calC| C_a^{-y-1}.
\end{align}
Setting the first-order derivative $\frac{\partial \underline U_a}{\partial {C_a}} = 0$. Then the maximal value of $\underline U_a$ is achieved when
\begin{align*}
    \what C_a = \left(\frac{y\eta_{c,a}D}{\eta_{p,a}c_b c_2(1-p)}\right)^{\frac{1}{p+y-1}}.
\end{align*}
It requires to compare $\underline U_a(\lfloor\what C_a\rfloor, S_{-a})$ and $\underline U_a(\lceil\what C_a\rceil, S_{-a})$ since the number of rounds for learning should be an integer. Then $\wtilde C_a$ is updated as the one achieving a larger payoff. That is,
\begin{align}
    \wtilde C_a = \arg\max_{C_a\in\{\lfloor\what C_a\rfloor, \lceil\what C_a\rceil\}} \underline U_a(C_a, S_{-a}).
\end{align}
To ensure $\wtilde C_a\le\tau$ ($\tau\ge 1$), it is required that
\begin{align}
    \what C_a \le \tau,
\end{align}
which leads to
\begin{align}
    \eta_{p,a}\cdot c_b c_2\cdot (1-p)\ge y\eta_{c,a} D{\tau}^{1 - p - y}. 
\end{align}
\end{proof}

\section{Analysis for Robust and Correlated Equilibrium}\label{app: analysis_for_robust_correlated_equilibrium}

Let $p_k^y$ represent the probability vector that player $k$ follows the instruction of the oracle, and $p_k^n$ represent the probability that the player does not follow the instruction of the oracle. For simplicity, we consider the case when there exist one defender and one attacker.

\begin{lem}
Let $E_d, E_a$ represent the strategy of the defender and the attacker separately, and $G_d, G_a$ represent the strategy of giving up using protection mechanism or attacking mechanism. The oracle draws one of $(E_d,G_a), (G_d,E_a)$, $(E_d,E_a)$ and $(G_d,G_a)$ with certain probabilities. Let $\hat U_d = \Psi_{[\underline U_k,\overline U_k]} U_d$, and $\hat U_a = \Psi_{[\underline U_k,\overline U_k]} U_a$. If 
\begin{align*}
    &(\hat U_d(G_d,G_a) - \hat U_d(E_d,G_a))\cdot\Pr[(G_d,G_a)] + (\hat U_d(G_d,E_a)- \hat U_d(E_d,E_a))\cdot \Pr[(G_d,E_a)]> 0, \\
    &(\hat U_d(E_d,E_a) - \hat U_d(E_d,G_a))\cdot\Pr[(E_d,E_a)] +   (\hat U_d(G_d,E_a) - \hat U_d(G_d,G_a))\cdot\Pr[(G_d,E_a)]> 0,\\
    &(\hat U_a(G_d,G_a) - \hat U_a(E_d,G_a))\cdot \Pr[(G_d,G_a)] + (\hat U_a(G_d,E_a) - \hat U_a(E_d,E_a)) \cdot \Pr[(G_d,E_a)] >0,\\
    &(\hat U_a(E_d,E_a) - \hat U_a(E_d,G_a))\cdot\Pr[(E_d,E_a)] +   (\hat U_a(G_d,E_a) - \hat U_a(G_d,G_a))\cdot\Pr[(G_d,E_a)]> 0,
\end{align*}
then the robust and correlated equilibrium is achieved when the players follow the instruction of the oracle with probability $1$. Specifically, the attacker follow the instructions of giving up attacking with probability $1$. 
\end{lem}
\begin{proof}
the robust and correlated equilibrium in the response strategy is achieved for player $k$ if he/she has no incentive to change the response probability $p_k^y$.


\begin{align*}
    p^R[(E_d,G_a)] =  p_a^{y}(G_a)\cdot p_d^{y}(E_d)\cdot \Pr[(E_d,G_a)] + p_a^{n}(E_a)\cdot p_d^{n}(G_d)\cdot \Pr[(G_d,E_a)] \\
    + p_a^{n}(E_a)\cdot p_d^{y}(E_d)\cdot \Pr[(E_d,E_a)] + p_a^{y}(G_a)\cdot p_d^{n}(G_d)\cdot \Pr[(G_d,G_a)].
\end{align*}

\begin{align*}
    p^R[(G_d,G_a)] =  p_a^{y}(G_a)\cdot p_d^{n}(E_d)\cdot \Pr[(E_d,G_a)] + p_a^{n}(E_a)\cdot p_d^{y}(G_d)\cdot \Pr[(G_d,E_a)] \\ 
    + p_a^{n}(E_a)\cdot p_d^{n}(E_d)\cdot \Pr[(E_d,E_a)] + p_a^{y}(G_a)\cdot p_d^{y}(G_d)\cdot \Pr[(G_d,G_a)].
\end{align*}

\begin{align*}
    p^R[(G_d,E_a)] =  p_a^{n}(G)\cdot p_d^{n}(E_d)\cdot \Pr[(E_d,G_a)] + p_a^{y}(E)\cdot p_d^{y}(G_d)\cdot \Pr[(G_d,E_a)] \\
    + p_a^{y}(E)\cdot p_d^{n}(E_d)\cdot \Pr[(E_d,E_a)] + p_a^{n}(G)\cdot p_d^{y}(G_d)\cdot \Pr[(G_d,G_a)].
\end{align*}

\begin{align*}
    p^R[(E_d,E_a)] =  p_a^{n}(G)\cdot p_d^{y}(E_d)\cdot \Pr[(E_d,G_a)] + p_a^{y}(E)\cdot p_d^{n}(G_d)\cdot \Pr[(G_d,E_a)] \\
    + p_a^{y}(E)\cdot p_d^{y}(E_d)\cdot \Pr[(E_d,E_a)] + p_a^{n}(G)\cdot p_d^{n}(G_d)\cdot \Pr[(G_d,G_a)].
\end{align*}

To ensure that the robust and correlated equilibrium is achieved when $p_a^{y}(G_a) = p_a^{y}(E_a) = p_d^{y}(G_d) = p_d^{y}(E_d) = 1$, it is required that the coefficient of $p_d^{y}(G_d)$ and $p_d^{y}(E_d)$ in the expected payoff of the defender should be positive, and the coefficient of $p_a^{y}(G_a)$ and $p_a^{y}(E_a)$ in the expected payoff of the attacker should be positive.

To ensure the robust and correlated equilibrium is achieved when $p_d^{y}(G_d) = 1$, the coefficient of $p_d^{y}(G_d)$ of the expected payoff of the defender should be positive, i.e.,
\begin{align}
    (\hat U_d(G_d,G_a) - \hat U_d(E_d,G_a))\cdot\Pr[(G_d,G_a)] + (\hat U_d(G_d,E_a)- \hat U_d(E_d,E_a))\cdot \Pr[(G_d,E_a)]> 0. 
\end{align}

To ensure that the robust and correlated equilibrium is achieved when $p_d^{y}(E_d) = 1$, the coefficient of $p_d^{y}(E_d)$ of the expected payoff of the defender should be positive, i.e.,

\begin{align}
    (\hat U_d(E_d,G_a) - \hat U_d(G_d,G_a))\cdot\Pr[(E_d,G_a)] + (\hat U_d(E_d,E_a) - \hat U_d(G_d,E_a))\cdot\Pr[(E_d,E_a)]> 0.
\end{align}

To ensure that the robust and correlated equilibrium is achieved when $p_a^{y}(G_a) = 1$, the coefficient of $p_a^{y}(G_a)$ of the expected payoff of the defender should be positive, i.e.,

\begin{align}\label{eq: slope_be_positive}
    (\hat U_a(G_d,G_a) - \hat U_a(G_d,E_a))\cdot \Pr[(G_d,G_a)] + (\hat U_a(E_d,G_a) - \hat U_a(E_d,E_a)) \cdot \Pr[(E_d,G_a)] >0.
\end{align}

To ensure that the robust and correlated equilibrium is achieved when $p_a^{y}(E) = 1$, the coefficient of $p_a^{y}(E)$ of the expected payoff of the defender should be positive, i.e.,

\begin{align}\label{eq: slope_a_y_e_be_positive}
    (\hat U_a(E_d,E_a) - \hat U_a(E_d,G_a))\cdot\Pr[(E_d,E_a)] +   (\hat U_a(G_d,E_a) - \hat U_a(G_d,G_a))\cdot\Pr[(G_d,E_a)]> 0.
\end{align}

In conclusion, if 
\begin{align*}
    &(\hat U_d(G_d,G_a) - \hat U_d(E_d,G_a))\cdot\Pr[(G_d,G_a)] + (\hat U_d(G_d,E_a)- \hat U_d(E_d,E_a))\cdot \Pr[(G_d,E_a)]> 0,\\
    &(\hat U_d(E_d,G_a) - \hat U_d(G_d,G_a))\cdot\Pr[(E_d,G_a)] + (\hat U_d(E_d,E_a) - \hat U_d(G_d,E_a))\cdot\Pr[(E_d,E_a)]> 0,\\
    &(\hat U_a(G_d,G_a) - \hat U_a(G_d,E_a))\cdot \Pr[(G_d,G_a)] + (\hat U_a(E_d,G_a) - \hat U_a(E_d,E_a)) \cdot \Pr[(E_d,G_a)] >0,\\
    &(\hat U_a(E_d,E_a) - \hat U_a(E_d,G_a))\cdot\Pr[(E_d,E_a)] +   (\hat U_a(G_d,E_a) - \hat U_a(G_d,G_a))\cdot\Pr[(G_d,E_a)]> 0,
\end{align*}
then the robust and correlated equilibrium is achieved when the players follow the instruction of the oracle with probability $1$. Specifically, the attacker follow the instructions of giving up attacking with probability $1$.
\end{proof}

With these constraints, we are now ready to solve this constrained optimization problem using KKT condition.

\begin{thm}\label{thm: kkt_special_solution_detail}
    Suppose $a_{11} a_{41} a_{32} a_{22} - a_{31} a_{12} a_{21} a_{42} \ne 0$. If any of the following conditions are satisfied:
 \begin{itemize} 
     \item $v_2 a_{22}=0$ and $v_4 a_{42}=0$ 
     \item $v_1 a_{12}=0$ and $v_4 a_{41}=0$ 
     \item $v_1 a_{12} = v_2 a_{22} = 0 $ and $v_4 a_{41} a_{42} = 0$
     \item $v_1 a_{12} = v_2 a_{22} = 0, a_{41}a_{42}v_4 \ne 0$ and $a_{41} a_{42} < 0$
 \end{itemize}
 then we can acquire that $x_1 = x_3 = 0$, $x_2 + x_4 = 1$, and $x_2, x_4 \ge 0$. Here, $v_1, v_2, v_3, v_4$ are represented by $c_k$ and $a_{ij}$, denoted in Equation \pref{eq: kkt_multiplier_solution_1}$\sim$\pref{eq: kkt_multiplier_solution_4}.
\end{thm}

\begin{proof}
The goal of the oracle is to ensure that the robust and correlated equilibrium for each player is to follow the instruction of the oracle with probability $1$, with a minimal amount of cost. 

Let $x_1 = \Pr[(G_d,G_a)], x_2 = \Pr[(G_d,E_a)], x_3 = \Pr[(E_d,G_a)]$, and $x_4 = \Pr[(E_d,E_a)]$. The cost is evaluated using $C = \sum_{i = 1}^4 c_i x_i$. We want to solve the following optimization problem. 

\begin{align}
\begin{array}{r@{\quad}l@{}l@{\quad}l}
\quad\min& \sum_{i = 1}^4 c_i x_i,\\
\text{s.t.,} & \text{any correlated equilibrim is a feasible solution of this LP}\\
& x_1 + x_2 + x_3 + x_4 = 1.
\end{array}
\end{align}

This optimization problem is further expressed as
\begin{align}
\begin{array}{r@{\quad}l@{}l@{\quad}l}
\quad\min& \sum_{i = 1}^4 c_i x_i,\\
\text{s.t.,} & a_{11}x_1 + a_{12}x_2 > 0\\
& a_{21}x_3 + a_{22}x_4 > 0\\
&a_{31}x_1 + a_{32}x_3 > 0\\
&a_{41}x_2 + a_{42}x_4 > 0\\
& x_1 + x_2 + x_3 + x_4 = 1\\
& x_1, x_2, x_3, x_4 \ge 0\\
\end{array}
\end{align}

This is a convex optimization problem with both inequality constraints and equality constraints. The optimal solution must meet the KKT condition. By rewriting the problem as:

\begin{align}
\begin{array}{r@{\quad}l@{}l@{\quad}l}
\quad\min& f(\mathbf{x})=\sum_{i = 1}^4 c_i x_i,\\
\text{s.t.,} & g_1(\mathbf{x}) = -(a_{11}x_1 + a_{12}x_2) \le 0\\
& g_2(\mathbf{x}) = -(a_{21}x_3 + a_{22}x_4) \le 0\\
& g_3(\mathbf{x}) = -(a_{31}x_1 + a_{32}x_3) \le 0\\
& g_4(\mathbf{x}) = -(a_{41}x_2 + a_{42}x_4) \le 0\\
& g_5(\mathbf{x}) = -x_1 \le 0, g_6(\mathbf{x}) = -x_2 \le 0\\
& g_7(\mathbf{x}) = -x_3 \le 0, g_8(\mathbf{x}) = -x_4 \le 0\\
& h(\mathbf{x}) = x_1 + x_2 + x_3 + x_4 - 1 = 0
\end{array}
\end{align}

We can form its Lagrangian function $L(\cdot)$ as:

\begin{equation}
    L(\mathbf{x}, \mu_1,\dots,\mu_8, \lambda) = f(\mathbf{x}) + \sum^{8}_{k=1} \mu_k g_k(\mathbf{x}) + \lambda h(\mathbf{x}), \\
    \mu_k \ge 0.
\end{equation}
where $\mu_k,\lambda$ are are Lagrange multipliers. We can further give its KKT condition for any optimal solution $\mathbf{x}^*$ and its corresponding multipliers. First, By applying $\frac{\partial L(\mathbf{x}, \mu_1,\dots,\mu_8,\lambda)}{\partial x_i} = \frac{\partial f}{\partial x_i} + \sum^{8}_{k=1}{\mu_k \frac{\partial g_k}{\partial x_i}} + \lambda\frac{\partial h_k}{\partial x_i} = 0$ and $\mu_k g_k(\mathbf{x})=0$, we get following new conditions:

\begin{align}
\begin{array}{r@{\quad}l@{}l@{\quad}l}
& c_1 - \mu_1 a_{11} - \mu_3 a_{31} -\mu_5 + \lambda = 0 \\
& c_2 - \mu_1 a_{12} - \mu_4 a_{41} -\mu_6 + \lambda = 0 \\
& c_3 - \mu_2 a_{21} - \mu_3 a_{32} -\mu_7 + \lambda = 0 \\
& c_4 - \mu_2 a_{22} - \mu_4 a_{42} -\mu_8 + \lambda = 0 \\
& \mu_1(a_{11}x_1 + a_{12}x_2) = 0\\
& \mu_2(a_{21}x_3 + a_{22}x_4) = 0\\
& \mu_3(a_{31}x_1 + a_{32}x_3) = 0\\
& \mu_4(a_{41}x_2 + a_{42}x_4) = 0\\
& \mu_5 x_1 = 0, \mu_6 x_2 = 0 \\
& \mu_7 x_3 = 0, \mu_8 x_4 = 0 \\
& \sum^4_{i=1}x_i = 1, (x_i \ge 0) \\
& \mu_j \ge 0 (j = 1,\dots,8)
\end{array}
\end{align}

Let's first consider a special case of setting $\mu_5=\mu_6=\mu_7=\mu_8=0$, $\lambda$ equals some fixed constant and was absorbed to $c_1,\dots,c_4$, and $\mu_1\mu_2\mu_3\mu_4 \ne 0$, then we can get $\mu_1,\mu_2,\mu_3,\mu_4$ by solving $\mathbf{A}\vec{\mu}=\mathbf{c}$:

\begin{equation}
    \begin{bmatrix}
a_{11} & 0 & a_{31} & 0 \\
a_{12} & 0 & 0 & a_{41} \\
0 & a_{21} & a_{32} & 0 \\
0 & a_{22} & 0 & a_{42}
\end{bmatrix}
\begin{bmatrix} \mu_1 \\ \mu_2 \\ \mu_3 \\ \mu_4 \end{bmatrix}
=
\begin{bmatrix} c_1 \\ c_2 \\ c_3 \\ c_4 \end{bmatrix}
\end{equation}
The equation has unique solution of:

\begin{align}
& v_1 = \frac{a_{31} a_{41} a_{21} c_4 - a_{31} a_{41} a_{22} c_3 - a_{31} a_{21} a_{42} c_2 + a_{41} a_{32} a_{22} c_1}{det(\mathbf{A})} \label{eq: kkt_multiplier_solution_1}\\
& v_2 = \frac{a_{11} a_{41} a_{32} c_4 - a_{11} a_{32} a_{42} c_2 - a_{31} a_{12} a_{42} c_3 + a_{12} a_{32} a_{42} c_1}{det(\mathbf{A})} \label{eq: kkt_multiplier_solution_2}\\
& v_3 = \frac{-a_{11} a_{41} a_{21} c_4 + a_{11} a_{41} a_{22} c_3 + a_{11} a_{21} a_{42} c_2 - a_{12} a_{21} a_{42} c_1}{det(\mathbf{A})} \label{eq: kkt_multiplier_solution_3}\\
& v_4 = \frac{a_{11} a_{32} a_{22} c_2 - a_{31} a_{12} a_{21} c_4 + a_{31} a_{12} a_{22} c_3 - a_{12} a_{32} a_{22} c_1}{det(\mathbf{A})} \label{eq: kkt_multiplier_solution_4}
\end{align}
when $det(\mathbf{A}) = a_{11} a_{41} a_{32} a_{22} - a_{31} a_{12} a_{21} a_{42} \ne 0$. Since the range of $a_{ij}$ is arbitrary, the above conditions can easily be satisfied to obtain feasible solutions for $\mu_i$ satisfying $v_i \ge 0$. After acquiring $v_i$, we can further obtain $\mathbf{x}$ by solving $\mathbf{A}^\top\mathbf{x}=\mathbf{0}$:

\begin{equation}
    \begin{bmatrix}
v_1 a_{11} & v_1 a_{12} & 0 & 0 & 0\\
0 & 0 & v_2 a_{21} & v_2 a_{22} & 0\\
v_3 a_{31} & 0 & v_3 a_{32} & 0 & 0\\
0 & v_4 a_{41} & 0 & v_4 a_{42} & 0\\
1 & 1 & 1 & 1 & 1
\end{bmatrix}
\begin{bmatrix} x_1 \\ x_2 \\ x_3 \\ x_4 \\ -1 \end{bmatrix}
=
\begin{bmatrix} 0 \\ 0 \\ 0 \\ 0 \\ 0\end{bmatrix}
\end{equation}\label{eq:linear_system_x}
The linear system admits solutions in numerous scenarios, contingent upon the ranges of $v_i a_{ij}$. Broadly, there exist two categories of solution sets:
\begin{itemize}[leftmargin=*]
    \item 0 zero: $x_i \ne 0$
    \item at least one zero: $x_1 x_2 x_3 x_4 = 0$
    \begin{itemize}
        \item 3 zeros: $x_i = 1, others=0$
        \item 2 zeros: $x_i+x_j=1, others=0$
        \item 1 zero: $x_i = 0, others=f(a_{ij})$
    \end{itemize}
\end{itemize}
Filtering feasible solutions satisfying $x_i \ge 0 $ is straightforward. To be more specific, let's consider the existence of a solution for the case $x_2=x_4=0, x_1+x_3=1$, where equation \ref{eq:linear_system_x} simplifies as:

\begin{align}
    & v_1 a_{11} \cdot x_1 = 0 \notag \\
    & v_2 a_{21} \cdot x_3 = 0 \notag \\
    & v_3 a_{31} \cdot x_1 + v_3 a_{32} \cdot x_3 = 0 \notag \\
    & x_1 + x_3 = 1 \notag
\end{align}

There are three cases for the solution of $x_1$ and $x_3$. For each case, the solution exists when the coefficients satisfy the corresponding constraints:

 \begin{itemize}[leftmargin=*]
     \item $\mathbf{x_1=0, x_3=1}$: $v_2 a_{21}=0, v_3 a_{32}=0$
     \item $\mathbf{x_1=1, x_3=0}$: $v_1 a_{11}=0, v_3 a_{31}=0$
     \item $\mathbf{x_1 x_3 \ne 0}$: in this case, the condition $v_1 a_{11} = v_2 a_{21} = 0$ must hold and if:
        \begin{itemize}
            \item $\mathbf{a_{31}a_{32}v_3 = 0}$: then $v_3 a_{31} = 0$ should hold when $a_{32} = 0$, $v_3 a_{32} = 0$ should hold when $a_{31} = 0$. The number of solution is infinite many.
            \item $\mathbf{a_{31}a_{32}v_3 \ne 0}$: then the solution is
            $$
            x_1 = \frac{a_{32}}{a_{32}-a_{31}}, x_3 = \frac{-a_{31}}{a_{32}-a_{31}}
            $$ when $a_{31} < 0 < a_{32}$ or $a_{32} < 0 < a_{31}$ holds.
        \end{itemize}
 \end{itemize}
 Since the range of $a_{ij}$ is arbitrary, the above conditions are easy to be satisfied.

\end{proof}

\section{Applications}\label{sec: applications_appendix}
Now we introduce some examples as illustrated in \pref{fig:attack_difficulty} concerning $0$-equilibrium, i.e., the attacker has no motivation to attack. Let the randomization mechanism be the protection mechanism, and DLG mechanism be the attacking mechanism. The parameters are set as: $x = 2, y = 1,  p = 1/2.$ The attacking cost for DLG is $C_a$, and is normalized as $E(C_a) = 1 - \frac{1}{(C_a)^y} = 1 - 1/C_a$. For randomization mechanism, the protection cost is $C(\Delta_k) = (\Delta_k)^x = (\Delta_k)^2$.










\begin{figure}
 \centering
 \subfigure[If the distortion extent of the defender is rather large, then the payoff of the attacker is always negative once he/she attacks]{
  \label{fig:a}
  \includegraphics[scale=0.67]{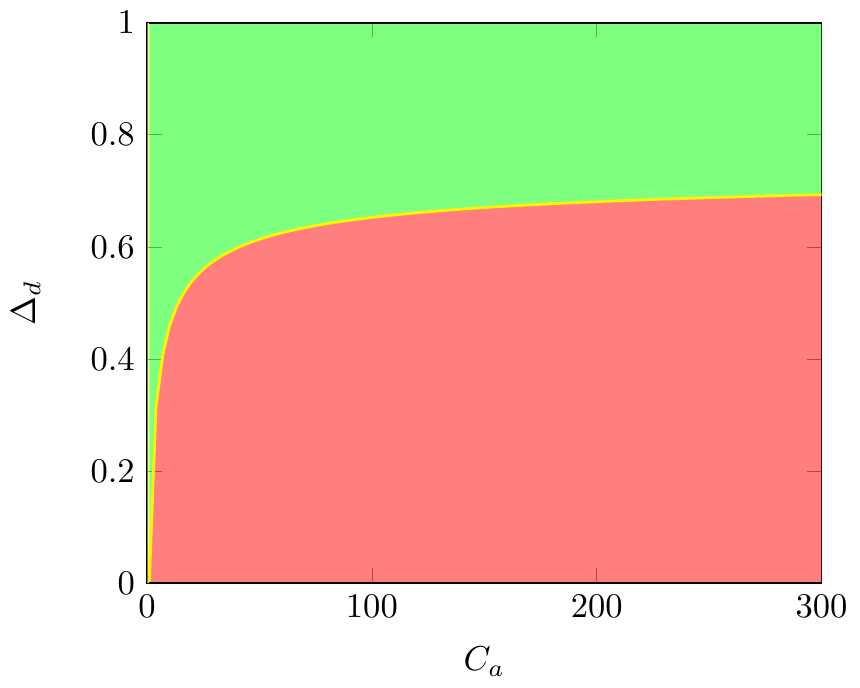}
 }
 \subfigure[For any distortion extent, the payoff of the attacker is always negative for any $C_a\ge 1$ and any $\Delta_d$.]{
  \label{fig:b}
  \includegraphics[scale=0.67]{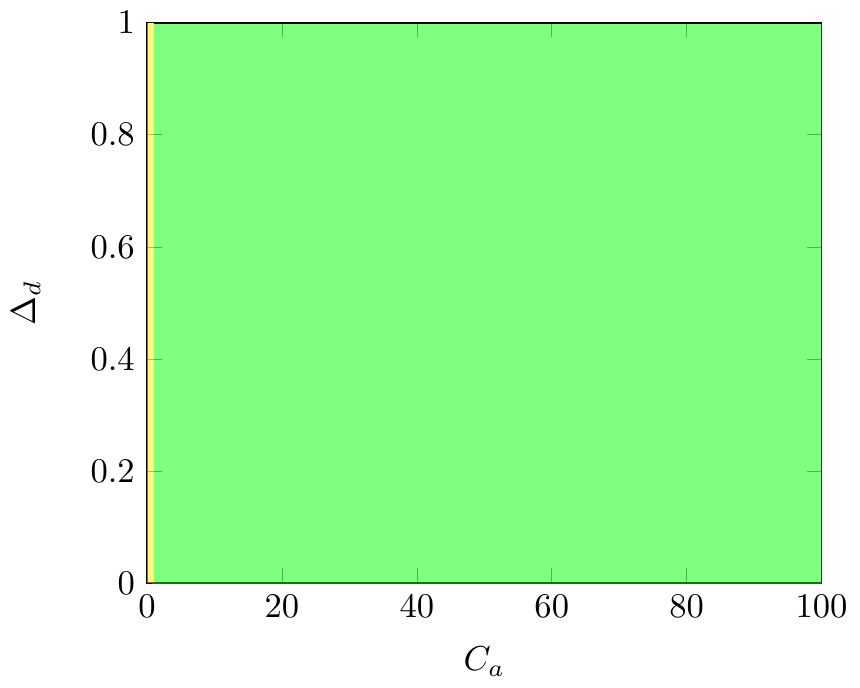}
 }
 \caption{
 Let $\epsilon = 0.0001$. $C_a(\in [0, 1/\epsilon] = [0,10,000])$ represents the number of rounds used by the attacker for inferring the private data of each defender using the optimization algorithm, $\Delta_d(\in [0,D] = [0,1])$ represents the distance between the exposed parameter information and the original information of the defender. The red area, yellow area and green area represent the region in which the payoff of the attacker is positive, zero and negative separately.
 \pref{fig:a}: Note that $x = 0.5\in (0,1)$ corresponds to protection mechanism with low efficiency such as HE. If $\Delta_d$ is rather large, then $\underline U_a(C_a, \Delta_d) <0$ for any $C_a\ge 1$. Therefore, the optimal attacking extent is $C_a^* = 0$, achieving the $0$-equilibrium in FLPG from our definition. \pref{fig:b}: $\underline U_a(C_a, \Delta_d)<0$ for any $C_a\ge 1$ and $\Delta_d\in [0,1]$. The robust equilibrium is a $0$-equilibrium in FLPG. It implies that the optimal strategy for the attacker is not to attack.
 }
 \label{fig:attack_difficulty}
 \vspace{-0.5mm}
\end{figure}

\pref{fig:a} provides a scenario when $0$-equilibrium in FLPG is achieved. Note that $x = 0.5\in (0,1)$ corresponds to protection mechanism with low efficiency such as the HE mechanism. When the protection extent $\Delta_d$ is rather large($\Delta_d\in [0,D] = [0,1]$), $\underline U_a(C_a, \Delta_d) <0$ for any $C_a\ge 1$. Therefore, the optimal attacking extent is $C_a^* = 0$, achieving the $0$-equilibrium in FLPG from our definition. The red area, yellow area and green area represent the region in which the payoff of the attacker is positive, zero and negative separately. \pref{fig:b} provides another scenario when $0$-equilibrium in FLPG is achieved. In this case, $\underline U_a(C_a, \Delta_d)<0$ for any $C_a\ge 1$ and $\Delta_d\in [0,1]$. The optimal strategy for the attacker is not to attack. The robust equilibrium is a $0$-equilibrium in FLPG.


\end{document}